\documentclass[journal]{IEEEtran}
\usepackage{bm}
\usepackage{booktabs}
\usepackage{color}
\usepackage{amsmath,amsfonts}
\usepackage{algorithmic}
\usepackage{algorithm}
\usepackage{array}
\usepackage[caption=false,font=normalsize,labelfont=sf,textfont=sf]{subfig}
\usepackage{textcomp}
\usepackage{stfloats}
\usepackage{url}
\usepackage{verbatim}
\usepackage{graphicx}
\usepackage{cite}
\newtheorem{theorem}{Theorem}
\newtheorem{definition}{Definition}

\newtheorem{proof}{Proof}
\usepackage{comment}

\begin{document}

\title{QFree: A Universal Value Function Factorization for Multi-Agent Reinforcement Learning}

\author{Rizhong~Wang,~Huiping~Li,~\IEEEmembership{Senior Member,~IEEE,}~Di~Cui and~Demin Xu
\thanks{R. Wang, H. Li, D. Cui and D. Xu are with School of Marine Science and Technology, Northwestern Polytechnical University, Xi'an 710072, China (e-mail: rizhongwang@mail.nwpu.edu.cn; lihuiping@nwpu.edu.cn; dicui@mail.nwpu.edu.cn; xudm@nwpu.edu.cn)}
\thanks{Manuscript received October 8, 2023; revised October 16, 2023.}}

\markboth{Journal of \LaTeX\ Class Files,~Vol.~14, No.~8, October~2023}%
{Shell \MakeLowercase{\textit{et al.}}: A Sample Article Using IEEEtran.cls for IEEE Journals}

\IEEEpubid{0000--0000/00\$00.00~\copyright~2023 IEEE}

\maketitle

\begin{abstract}
Centralized training is widely utilized in the field of multi-agent reinforcement learning (MARL) to assure the stability of training process. Once a joint policy is obtained, it is critical to design a value function factorization method to extract optimal decentralized policies for the agents, which needs to satisfy the individual-global-max (IGM) principle. While imposing additional limitations on the IGM function class can help to meet the requirement, it comes at the cost of restricting its application to more complex multi-agent environments. In this paper, we propose QFree, a universal value function factorization method for MARL. We start by developing mathematical equivalent conditions of the IGM principle based on the advantage function, which ensures that the principle holds without any compromise, removing the conservatism of conventional methods. We then establish a more expressive mixing network architecture that can fulfill the equivalent factorization. In particular, the novel loss function is developed by considering the equivalent conditions as regularization term during policy evaluation in the MARL algorithm. Finally, the effectiveness of the proposed method is verified in a nonmonotonic matrix game scenario. Moreover, we show that QFree achieves the state-of-the-art performance in a general-purpose complex MARL benchmark environment, Starcraft Multi-Agent Challenge (SMAC).
\end{abstract}


\section{Introduction}

\IEEEPARstart{M}{ulti-agent} cooperative systems are widely used in the fields such as military, sensor networks and autonomous driving \cite{Cao2013, Panait2005, Busoniu2008,Li2023}. Reinforcement learning is a promising solution to enhance the intelligent level of multi-agent cooperative systems, which has shown great success in game decision making \cite{Silver2017}, robot control \cite{Levine2016}, and biological protein structure prediction \cite{Tunyasuvunakool2021}. In practical implementation of multi-agent reinforcement learning (MARL), however, the environment state of the individual agent in multi-agent systems is sensitive to the actions of other agents, making the training process unstable \cite{Tuyls2017}. Moreover, practical constraints such as limited sensor capabilities make it impossible for the agents to perceive global information, which can intensify the instability of training process \cite{Oliehoek2016}. One approach to tackle this instability is through centralized training, which takes individual agents as whole and combines their observable information, eliminating mutual influence between agents \cite{Lowe2017}.

Centralized training introduces the following new challenges  \cite{Yang2020,Foerster2022}. (i) As the number of agents in a multi-agent system grows, the exponentially increased action space dimension for optimization will lead to the ``dimensional explosion" phenomenon, posing a great challenge to cope with the training complexity; (ii) in a centralized training setting, individual agents are unable to discern their own independent reward values and comprehend their respective contributions to the overall reward optimization process. Consequently, this can lead to the emergence of ``lazy agents" who fail to actively contribute to the collective objective.
\IEEEpubidadjcol

Centralized training and decentralized execution (CTDE) has emerged as a promising solution to address the above challenges faced in multi-agent systems  \cite{Kraemer2016,Oliehoek2018,Hong2022,Tan2014,Hu2023}. In this approach, the joint value function is trained in a centralized manner, which allows the agent to share observed information and remove interaction between agents during the training process. After that, following the Individual-global-max (IGM) principle, each agent constructs its individual valued function by factorizing the joint one.  The ``dimensional explosion" issue is avoided because the optimal action sequence for the agents produced by taking the maximizing operation on the global joint value function is guaranteed to be equivalent to the one optimized by their individual value function. Furthermore, the individual value function encourages agents to actively participate in achieving the collective objective to eliminate ``lazy agents". Agent then operates independently using only its own localized observations and the learned individual valued function. This approach has provided satisfactory results in classical MARL benchmark environments like the Starcraft Multi-Agent Challenge (SMAC) and demonstrates its effectiveness in various practical scenarios, showing promising potential to be applied in MARL \cite{Samvelyan2019}.

Note that finding a value function factorization method satisfying the IGM principle is critical in CTDE-based algorithms, which ensures the joint value function of the multi-agent system is consistent with the factorized individual one during policy evaluation. Though the conditions in the principle usually hold by imposing constraints on the properties of value function, most practical complex scenarios cannot meet the restricted properties, which motivates this study. In this paper, a universal value function factorization method called QFree in the framework of CTDE is proposed for the cooperation of multi-agent systems. The main contributions can be summarized as follows.
\begin{itemize}
	\item A novel advantage function based method is designed to develop mathematical equivalent formulation of the IGM principle. This method transforms the original conditions in IGM principle into a universal form, which is a completely equivalent factorization without imposing any constraints on the factorized advantage functions, and can remove the conservatism. We have theoretically proved that this is a necessary and sufficient factorization.
	
	\item A new mixing network architecture and an MARL algorithm are designed to fulfill the IGM principle. The network enables the factorization of the joint value function for all kinds of IGM function class, and the new loss function is designed by considering the equivalent conditions in IGM principle as regularization term.
	
	\item The proposed method demonstrates state-of-the-art performance compared to other algorithms in the challenging MARL complex environment SMAC. 
\end{itemize}

The rest of this paper is organized as follows: Section \ref{Related Work} describes the related work in the field of MARL. Section \ref{Preliminaries} presents the fundamentals of MARL. The proposed algorithm is designed in detail in Section \ref{Introducing my algorithm}. The experiments and comparison studies are presented in Section \ref{Experiments}. We present conclusions in Section \ref{Conclusion}.

\section{Related Work}
\label{Related Work}
Cooperative MARL originally relies on tabular representation \cite{Busoniu_2006}, which is only suitable for the system with small state and action spaces. To address the challenges imposed by high-dimensional state and action spaces, the method integrating deep learning into MARL has been developed \cite{Vincent_2018}.

Similar to single-agent reinforcement learning algorithms, MARL algorithms can be categorized into two main types: value function-based methods and policy gradient-based methods \cite{Zhang_2021}. Within the policy gradient-based approach, a popular framework is the actor-critic method. In this framework, each agent shares the same critic network to learn the value function and improve the policies. Because each agent has its own actor network, it is capable of using CTDE fashion. Following this framework, multi-agent deep deterministic policy gradient (MADDPG) method was designed, which extends DDPG, the classical reinforcement learning algorithm for the single agent, into the multiple case \cite{Song_2021}. Considering the superior performance of the proximal policy optimization (PPO) algorithm, \cite{Yu_2022} further developed multi-agent PPO (MAPPO) for MARL. Additionally,  \cite{Iqbal_2018} designed the multi-agent actor-attention-critic (MAAC) algorithm, which introduces the attention mechanism into the actor-critic framework to improve the performance of MARL. To cope with the ``lazy agents",  counterfactual multi-agent policy gradients (COMA) algorithm is developed in \cite{Foerster2022}, which employs counterfactual baseline technology within the actor-critic framework. Note that the input space of the shared critic network will increase with the number of agents, leading to the ``dimensional explosion" in policy gradient-based methods.

The earliest value function-based method is independent Q-learning (IQL), which aims to make the agent in multi-agent systems independent of the learning policy \cite{Tampuu_2017}. In order to improve the learning efficiency and stability, value function factorization method based on CTDE framework has become a commonly adopted approach in value function-based MARL. To meet the IGM principle, value-decomposition networks (VDN) algorithm \cite{Sunehag2018} directly represented the joint value function by a sum of the individual value functions. On this basis, the attention mechanism is introduced into VDN algorithm to adjust the coefficient in the sum \cite{Wei_2022}. Compared with VDN algorithm, QMIX proposed in \cite{Rashid2020} relaxed the sum constraints by considering the joint value function as a monotonic function of the individual value functions using the hypernetwork approach. Furthermore, the work in \cite{Wang2021} employed the duplex dueling structure to replace the value function in IGM principle by the advantage function, reducing the conservatism of the original principle. WQMIX and OW-MIX adopt distinct weighting schemes for the joint value function using the temporal difference (TD) errors, allowing for a more relaxed adherence to the monotonicity constraint imposed by QMIX method \cite{Rashid2021}. While all of the above methods have improved the structure of the mixing network by designing sufficient conditions to satisfy the IGM principle but not the necessary and sufficient conditions \cite{Hu2021}. Though QTRAN achieves a complete factorization that satisfies the IGM principle by utilizing the three value networks \cite{Son2019}, it does not perform well in complex environments because simultaneously optimizing three networks are challenging. Consequently, effectively  developing the value function factorization method without restricting value function's class remains an open challenge.

\section{Preliminaries}
\label{Preliminaries}

\subsection{Dec-POMDP and Deep Q-learning}
\label{Dec-POMDP and Deep Q-learning}

We first model MARL with a decentralized partially observable Markov decision process (Dec-POMDP), which can be represented by a tuple $G = \{N, S, A, P, r, Z, O,\gamma\}$ \cite{Ong_W2016,Velez2015}. Here, $N=\{1,...,n\}$ is the collection of $n$ agents. $s \in S$ denotes the global state of the multi-agent system's environment. Agent $i \in N$ will choose its own action $a_i \in A$ at each time step, and the joint action is represented by $\bm{a}=[a_i]_{i=1}^n \in A^N$. The system dynamics is defined as $P(s'|s,a):S \times A \times S \rightarrow [0,1]$, where $s'$ is the global state at next step. Because of the limited observability, individual agents can only observe partial environmental information $o \in O$, which is taken as an unknown function of global state $s$ with $o_i = O(s, i)$. The global reward function of the multi-agent system is $r(s,\bm{a}):S \times A^N \rightarrow R\in\mathbb R$, and $\gamma \in [0,1)$ is the discount factor.

In reinforcement learning, we define $\pi(a_t|s_t)$ as the stochastic policy of the agent at state $s_t$, and the agent maximizes the return $G_t=R_{t+1} + \gamma R_{t+2} + \gamma^2 R_{t+3} + ...$ at current moment with respect to the policy $\pi(a_t|s_t)$. In order to evaluate the policy, we introduce the action value function: $Q_\pi(s_t,a_t)=\mathbb{E}[G_t|s_t,a_t]$. The optimal policy can be found by maximizing the optimal action value function $Q^*_\pi(s_t,a_t)$ \cite{Sutton_Barto1998}:
\begin{equation}\label{optimal policy}
	\pi^*(a_t|s_t)=
	\begin{aligned}
			\begin{cases}
			1 & a_t=\mathop{\arg\max}\limits_{a \in A}Q^*_\pi(s_t, a),\\
			0 & otherwise.
		\end{cases}
	\end{aligned}
\end{equation}
In a partially observable environment, the multi-agent system uses $\bm z_t=[o_{t-k},a_{t-k};o_{t-k+1},a_{t-k+1},...,o_t,a_t]$, the local action-observation history information of the last $k$ steps, to estimate $Q_\pi(s_t,a_t)$, i.e., $Q_\pi(s_t,a_t) \approx Q_\pi(\bm z_t,a_t)$.

The deep Q-network (DQN) algorithm employs a deep neural network to approximate the optimal action value function $Q^*_\pi(s_t,a_t)$. By leveraging the dataset $<s, s', r, a>$ derived from the agent's exploration of the environment, the parameter $\bm \theta$ of DQN is iteratively optimized. The training loss function of DQN is determined based on the TD error in reinforcement learning:
\begin{equation}\label{loss function}
	\begin{aligned}
		\mathcal{L}(\bm \theta)=(r+\gamma\mathop{\max}\limits_{a' \in A(s')}Q(s',a';\bm \theta^-)-Q(s,a;\bm \theta))^2,
	\end{aligned}
\end{equation}
where $\bm \theta^-$ is the parameters of the target network, which shares the same structure as the Q-network and performs synchronized updates with the Q-network parameters $\bm \theta$ at regular intervals. This design choice aims to ensure training stability and optimize the training target, as suggested in the literature \cite{Mnih2015}.
\subsection{CTDE and IGM Principle}\label{CTED and IGM Condition}

During the training process, both the joint local action-observation history $\bm z$ and the global state $s$ can be utilized. However, when it comes to the execution of the learned policy, each individual agent solely relies on its local action-observation histories $z_i$, which is the core idea of CTDE. One challenge in the execution phase of the learned policy, in adherence to CTDE, is to establish the optimal consistency between the joint action value function, denoted by $Q_{tot}(\bm z, \bm a)$, in the training process and the individual agent action value function, denoted as $Q_i(z_i,a_i)$, which is further mathematically defined by the following IGM principle \cite{Rashid2020}.
\begin{definition}
	For the joint action value function $Q_{tot}(\bm z, \bm a)$, if the equality
	\begin{equation}\label{IGM}
		\mathop{\arg\max}\limits_{\bm a \in A^N}Q_{tot}(\bm z, \bm a)=
		\begin{pmatrix}
			{\arg\max}_{a_1 \in A}Q_{1}(z_1, a_1)\\
			\vdots\\
			{\arg\max}_{a_n \in A}Q_{n}(z_n, a_n)\\
		\end{pmatrix}
	\end{equation}
 holds, we call $Q_{tot}$ can be factorized by $[Q_i]_{i=1}^n$ with the satisfaction of IGM principle.
\end{definition}

Once the IGM principle holds, the factorization of the trained optimal state-action value function is straightforward. In order to satisfy the IGM principle, many methods have been proposed. For example, the VDN algorithm \cite{Sunehag2018} is a pioneering work in the field of value function factorization method. It represents the joint action value function as a sum of individual action value functions, that is,
\begin{equation}\label{VDN}
	\begin{aligned}
		Q_{tot}(\bm z, \bm a)=\sum_{i=0}^nQ_{i}(z_i, a_i).
	\end{aligned}
\end{equation}
The QMIX algorithm in \cite{Rashid2020} constructs the mapping relation $Q_{tot}(\bm z, \bm a)=f(Q_{1}(z_1, a_1),...,Q_{n}(z_n, a_n))$ between $Q_i$ and $Q_{tot}$ to make it satisfy the following monotonicity condition
\begin{equation}\label{QMIX}
	\begin{aligned}
		\frac{\partial Q_{tot}(\bm z, \bm a)}{\partial Q_{i}(z_i, a_i)}\geq0 \quad i \in N,
	\end{aligned}
\end{equation}
further relaxing the requirement of the IGM principle.
By factorizing action value function $Q_{tot}$ into state value function $V_{tot}$ and advantage function $A_{tot}$
\begin{equation}\label{QPLEX}
	\begin{aligned}
		Q_{tot}(\bm z, \bm a)=V_{tot}(\bm z)+A_{tot}(\bm z, \bm a),
	\end{aligned}
\end{equation}
DuPLEX dueling multi-agent Q-learning (QPLEX) in \cite{Wang2021} transforms the monotonicity constraint on $Q_{tot}$ to advantage function $A_{tot}$. 

However, these value function factorization methods only provide sufficient but not necessary conditions for IGM principle, and the monotonicity constraint limits the class of the joint action value function $Q_{tot}$. When applying them to the non-monotonic case, these methods cannot learn the correct joint action value function. Specifically, as the matrix game we show in Table \ref{tab1} (a), it represents the cooperation of two agents. Table \ref{tab1} (a) is the real value of $Q_{tot}$, which determines the specific rules for the game \cite{Hu2021}. Denote the action space of the two agents as $A=\{0,1,2\}$. When both agents adopt action $0$ at the same time, they will receive the maximum reward $1$. When only one of them adopts action $0$, the reward is $-12$. Otherwise, the reward is $0$.  Note that for agent $1$, when it takes action $0$, the individual action value function $Q_1(z_1, 0)=1\times\frac{1}{3}+(-12)\times\frac{1}{3}+(-12)\times\frac{1}{3}\approx -7.7$. Similarly, we have $Q_1(z_1, 1)=-4$ and $Q_1(z_1, 2)=-4$. For agent $2$, it can be obtained that $Q_2(z_2, 0)=-7.7$, $Q_2(z_2, 1)=-4$ and $Q_2(z_2, 2)=-4$, respectively. In this case, the actual joint action value function $Q_{tot}$ in Table \ref{tab1} (a) does not increase with the increase of $Q_i$. However, as shown in Table \ref{tab1} (b), when using QMIX learns this non-monotonic relationship, only the monotonic results are obtained, failing to learn the correct $Q_{tot}$ in Table \ref{tab1} (a).

\begin{table}
	\caption{Matrix game (red bold is optimal reward).}
	\label{tab1}
	\centering
	\subfloat[The real $Q_{tot}$.]{
		\label{tab1a}
		\begin{tabular} {|c|c|c|c|}
			\hline
			& & & \\[-8pt]
			$\bm a$ & 0 & 1 & 2\\
			\hline
			& & & \\[-8pt]
			0 & {\color{red}\textbf{1}} & -12& -12  \\
			\hline 
			& & & \\[-8pt]
			1 & -12 & 0 & 0 \\
			\hline
			& & & \\[-8pt]
			2 & -12 & 0 & 0 \\
			\hline 
		\end{tabular}
		
}
	\subfloat[QMIX: $Q_{tot}$.]{
		\label{tab1b}
		\centering
		\begin{tabular}{|c|c|c|c|}
			\hline
			& & & \\[-8pt]
			$\bm a$ & 0 & 1 & 2\\
			\hline
			& & & \\[-8pt]
			0 & -9.4 & -9.4& -9.4  \\
			\hline 
			& & & \\[-8pt]
			1 & -9.4 & 0 & 0 \\
			\hline
			& & & \\[-8pt]
			2 & -9.4 & 0 & {\color{red}\textbf{0}} \\
			\hline 
		\end{tabular}
		
	}
\end{table}

\section{Universal Value Function Factorization Method}
\label{Introducing my algorithm}

In this section, we propose the QFree algorithm, which serves as algorithmic framework for developing lossless joint action value functions to satisfy the IGM principle. Initially, we factorizing the action value function into state value function and advantage function using dueling network architecture. Subsequently, we transform the value function based IGM principle into a novel advantage function based IGM principle. Through designing a new mixing network architecture and learning algorithm, the advantage function based IGM principle is translated into a mathematical constraint that is incorporated into the optimization process using regularization.

\subsection{New Advantage Function Based IGM Principle}
\label{Advantage-based IGM Principles}

The action value function $Q$ depends on both the state $s$ and the action $a$ and it reflects theirs' joint impact. However, in certain situations, evaluating the quality of the current state $s$ becomes more significant than choosing the  action $a$. Conversely, in other scenarios, the influence of the action $a$ is the main concern. In this case, it is common to split the action value function $Q$ into two separate components: the state value function $V$ is evaluated only for state $s$ and the advantage function $A$ focuses on action $a$, satisfying $Q = V + A$ \cite{10.5555/3045390.3045601}. Naturally, the joint action value function $Q_{tot}$ and the individual action value function $Q_i$ of agents can also be written in this form as follows.
\begin{equation}\label{Advantage}
	\begin{aligned}
		Q_{tot}(\bm z, \bm a)=V_{tot}(\bm z)+A_{tot}(\bm z, \bm a),\\
		Q_{i}(z_i, \bm a_i)=V_{i}(z_i)+A_{i}(z_i, a_i).
	\end{aligned}
\end{equation}

In (\ref{Advantage}), simultaneously changing the state value function $V$ and the advantage function $A$ will lead to the summed action value function $Q$ unchanged, thereby making it hard to find the optimal policy.  To address this issue, a standard method is setting the advantage function as zero under the optimal action $a^*$ \cite{10.5555/3045390.3045601}. Then, we have
\begin{equation}\label{zero}
	\begin{aligned}
		V_{tot}(\bm z)=\mathop{\max}\limits_{\bm a \in A^N}Q_{tot}(\bm z, \bm a),\\
		V_{i}(z_i)=\mathop{\max}\limits_{a_i \in A}Q_{i}(z_i, a_i).
	\end{aligned}
\end{equation}

Since the state value function $V_{tot}(\bm z)$ and $V_{i}(z_i)$ are independent of the action selection, the state value function $V_{tot}(\bm z)$ and $V_{i}(z_i)$ does not need to satisfy the IGM principle. In this way, one can transform the IGM principle from the action value function into the advantage function based IGM principle as follows.
\begin{definition}\label{definition2}
	For a joint advantage function $A_{tot}$, if there exists an independent set of advantage functions $[A_i]_{i=1}^n$ satisfying:
	\begin{equation}\label{advantage-based IGM}
		\mathop{\arg\max}\limits_{\bm a \in A^N}A_{tot}(\bm z, \bm a)=
		\begin{pmatrix}
			{\arg\max}_{a_1 \in A}A_{1}(z_1, a_1)\\
			\vdots\\
			{\arg\max}_{a_n \in A}A_{n}(z_n, a_n)\\
		\end{pmatrix},
	\end{equation}
	then we claim that $A_{tot}$ can be factorized by $[A_i]_{i=1}^n$ satisfying the advantage function based IGM principle.
\end{definition}

The equivalence of the two principles are verified in \cite{Wang2021}. In the following, we propose an equivalent condition for the IGM principle in (\ref{advantage-based IGM}) in the following Theorem \ref{theorem1}.
 
\begin{theorem}\label{theorem1}
	For the advantage function based IGM principle in (\ref{advantage-based IGM}), let $a^*_i=\mathop{\arg\max}\limits_{a_i \in A}A_{i}(z_i, a_i)$ and $\bm a^*=[a^*_1,a^*_2,...,a^*_n]$. Then the principle (\ref{advantage-based IGM}) holds if and only if the following conditions are satisfied:
	\begin{equation}\label{constraint}
		\begin{aligned}
	\begin{cases}
		A_{tot}(\bm z, \bm a)\leq0 & \bm a\neq\bm a^*,\\
		A_{tot}(\bm z, \bm a)=0 & \bm a=\bm a^*.
	\end{cases}
\end{aligned}
	\end{equation}
\end{theorem}

The complete proof of Theorem \ref{theorem1} is shown in Appendix \ref{Appendix A1}. Note that Theorem 1 not only provides an equivalent form of IGM principle, but also develops necessary and sufficient conditions on value function factorization. Through these conditions, we can realize the value factorization without any conservatism. In the subsequent subsection, we elaborate on how to satisfy these conditions in (\ref{constraint}) through regularization, and design detailed mixing network architecture to satisfy them.

\subsection{The Algorithm Architecture}
\label{The Algorithm Architectures}

The implementation procedure of the entire algorithm will be thoroughly discussed in this subsection. First, we will provide a detailed description of the precise architecture of the mixing network designed for the proposed algorithm. The primary function of this mixing network is to convert the independent action value function into a joint action value function. Once training is completed, during the specific execution phase, we will eliminate the components of the mixing network and decentralize execution by utilizing their respective independent action value functions. Subsequently, we will introduce and analyze how to effectively employ the proposed value function factorization method, thus achieving Theorem \ref{theorem1} without additional requirements on the monotonicity of the value function, along with completely presenting our learning algorithm.

\subsubsection{Mixing Network Architecture}
\label{Mixing Network Architectures}

The input to the mixing network comprises the observation and action information of each agent, while the output is the joint action value function required for MARL. In what follows, we will detail the structure of the network.

\begin{figure*}[!t]
	\centering
	\includegraphics[width=0.8\linewidth]{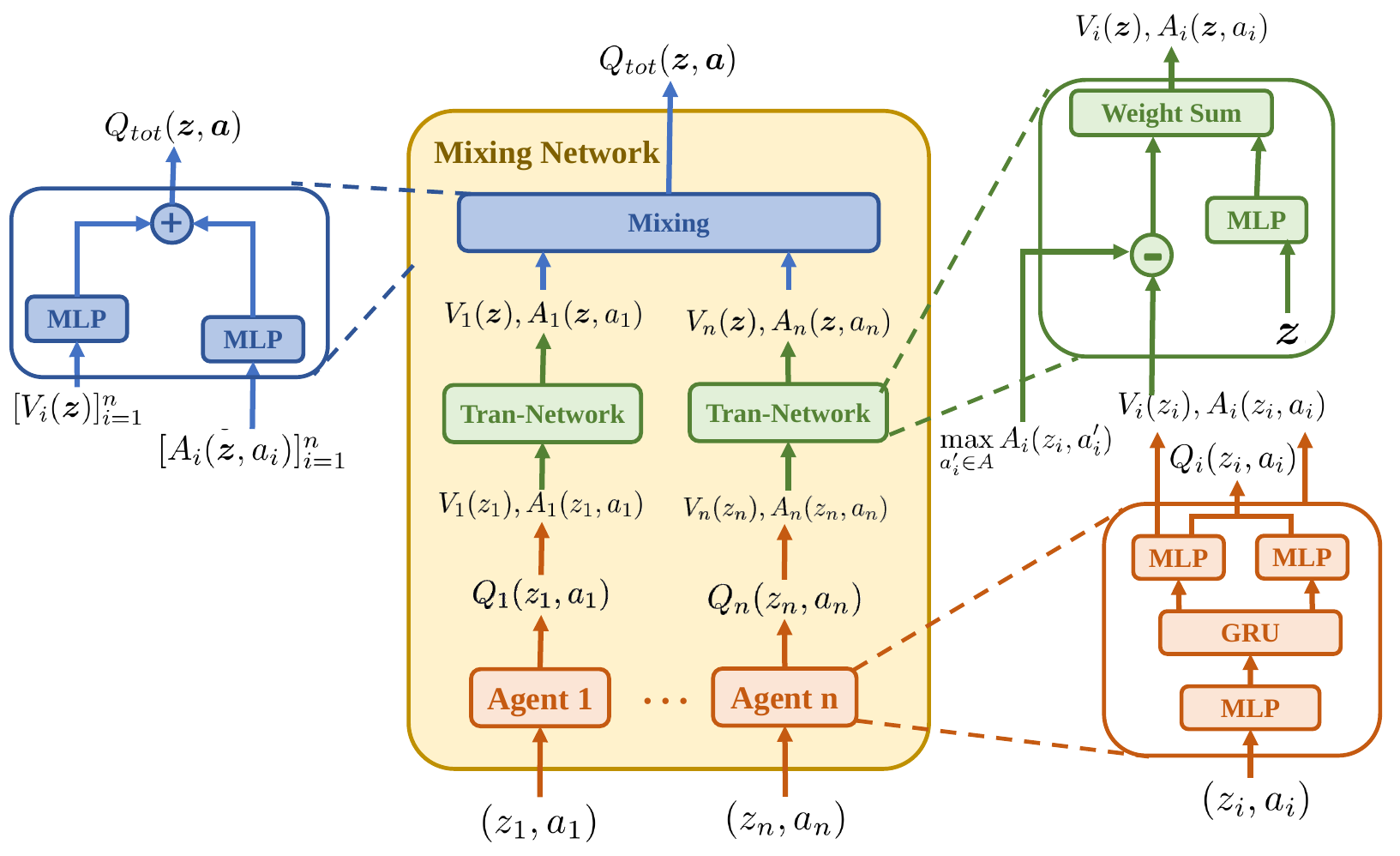}
	\caption{Mixing network architecture.}
	\label{mixing network architectures}
\end{figure*}

Due to the partial observability of the environment, the agent is unable to acquire complete state information. To mitigate its impact on the reinforcement learning algorithm, we utilize the historical observation sequence $z_i$ for each agent as a substitute for state $s_i$. By employing recurrent neural networks (RNN) with hidden states $h_i$, we effectively exploit these historical observation sequences for learning purposes. Consequently, we replace the fully connected feedforward neural network in the Q-network with an RNN architecture. Subsequently, behind this RNN framework, two fully connected neural networks are designed to respectively output the state value function $V_{i}(z_i)$ and advantage function $A_{i}(z_i,a_i)$. The action value function $Q_{i}(z_i,a_i)$ is then obtained by combining $V_{i}(z_i)$ and $A_{i}(z_i,a_i)$ through equation (\ref{Advantage}). To ensure that $A_{i}(z_i, a_i^*)=0$ as required in (\ref{zero}), we set
\begin{equation}\label{trans}
	\begin{aligned}
		A_{i}(z_i, a_i)\leftarrow A_{i}(z_i, a_i)-\mathop{\max}\limits_{a'_i \in A}A_{i}(z_i, a'_i).
	\end{aligned}
\end{equation}

Then, we use a transformation network \cite{Wang2021} to transform the state value function $V_{i}(z_i)$ and the advantage function $A_{i}(z_i, a_i)$ of each agent into $V_{i}(\bm z)$ and $A_{i}(\bm z, a_i)$. The specific expression is as follows.
\begin{equation}\label{trans_1}
	\begin{aligned}
		V_{i}(\bm z)=\omega_i(\bm z)V_{i}(z_i)+b_i(\bm z),\\
		A_{i}(\bm z, a_i)=\omega_i(\bm z)A_{i}(z_i, a_i).
	\end{aligned}
\end{equation}
Here, $\omega_i(\bm z)$ and $b_i(\bm z)$ represent the constants generated by the feedforward neural network, with $z$ being the input.

Next, we change $V_{i}(\bm z)$ and $A_{i}(\bm z, a_i)$ into joint state value function $V_{tot}(\bm z)$ and joint advantage function $A_{tot}(\bm z, \bm a)$ through two feedforward mixing networks, and the expression is
\begin{equation}\label{trans_2}
	\begin{aligned}
		V_{tot}(\bm z)=f_v([V_{i}(\bm z)]_{i=1}^n,\theta_v),\\
		A_{tot}(\bm z, \bm a)=f_a([A_{i}(\bm z, a_i)]_{i=1}^n,\theta_a).
	\end{aligned}
\end{equation}
Here, $f_v(.)$ and $f_a(.)$ denote the mappings of two feedforward neural networks, with the input $[V_{i}(\bm z)]_{i=1}^n=[V_{1}(\bm z),V_{2}(\bm z),...,V_{n}(\bm z)]$, $[A_{i}(\bm z, a_i)]_{i=1}^n=[A_{1}(\bm z, a_i),A_{2}(\bm z, a_i),...,A_{n}(\bm z, a_i)]$. The parameters of these networks are denoted by $\theta_v$ and $\theta_a$, respectively.

Finally, the joint action value function $Q_{tot}(\bm z, \bm a)$ is obtained by combining  $V_{tot}(\bm z)$ and $A_{tot}(\bm z, \bm a)$ according to (\ref{Advantage}). The complete architecture of the mixing network for our approach is illustrated in Fig. \ref{mixing network architectures}, providing a detailed process to get the joint action value function $Q_{tot}(\bm z, \bm a)$ from the input $[(z_i, a_i)]_{i=1}^n$.

\subsubsection{Policy Evaluation Algorithm}
\label{Policy Evaluation Algorithm}

In this subsection, we propose a novel MARL algorithm and provide a detailed description of its specific process. The entire process of reinforcement learning algorithm can be regarded as the solution of a Bellman optimal equation, and the solution process is a Bellman optimal operator. It involves two components: policy evaluation and policy improvement. To enhance the agent's exploration during the learning process, we adopt the $\epsilon$-greedy policy of (\ref{optimal policy}) for policy improvement
\begin{equation}\label{y_optimal policy}
	\pi^*(a_t|s_t)=
	\begin{aligned}
		\begin{cases}
			1-\epsilon & a_t=\mathop{\arg\max}\limits_{a \in A}Q^*_\pi(s_t,a),\\
			\epsilon & otherwise.
		\end{cases}
	\end{aligned}
\end{equation}
Here, $\epsilon \in [0,1)$ and decreases with the number of training episodes. According to (\ref{loss function}), the policy evaluation in MARL evaluates the joint policy by the following loss function
\begin{equation}\label{loss function_j}
	\begin{aligned}
		\mathcal{L}_{td}(\bm \theta)=(r+\gamma\mathop{\max}\limits_{a' \in A(s')}Q_{tot}(\bm z, \bm a;\bm \theta^-)-Q_{tot}(\bm z, \bm a;\bm \theta))^2,
	\end{aligned}
\end{equation}
where $\bm \theta$ is the whole mixing network parameters, including each agent Q-network, tran-network and mixing network. $\bm \theta^-$ is the target network parameter, which has the same structure as the mixing network. The policy evaluation process can be regarded as a parameter optimization problem as follows:
\begin{equation}\label{loss function_o}
	\begin{aligned}
		&\mathop{\arg\min}\limits_{\bm \theta}\mathcal{L}_{td}(\bm \theta)=\\
		&\mathop{\arg\min}\limits_{\bm \theta}(r+\gamma\mathop{\max}\limits_{a' \in A(s')}Q_{tot}(\bm z, \bm a;\bm \theta^-)-Q_{tot}(\bm z, \bm a;\bm \theta))^2.
	\end{aligned}
\end{equation}
According to Theorem \ref{theorem1}, the optimization problem (\ref{loss function_o}) can be transformed into a constrained optimization problem:
\begin{equation}\label{loss function_c}
	\begin{aligned}
		\bm \theta^*=&\mathop{\arg\min}\limits_{\bm \theta}(r+\gamma\mathop{\max}\limits_{a' \in A(s')}Q_{tot}(\bm z, \bm a;\bm \theta^-)-Q_{tot}(\bm z, \bm a;\bm \theta))^2\\
		s.t.\quad& A_{tot}(\bm z, \bm a;\bm \theta)\leq0 \quad  if\ \bm a\neq\bm a^*,\\
		& A_{tot}(\bm z, \bm a^*;\bm \theta)=0.
	\end{aligned}
\end{equation}
Here, $\bm a^*=[a_1^{max}(z_1';\bm \theta),a_2^{max}(z_2';\bm \theta),...,a_n^{max}(z_n';\bm \theta)]$ is obtained by optimizing the following individual Q-network
\begin{equation}\label{a_max}
	\begin{aligned}
		a_i^{max}(z_i';\bm \theta)=\mathop{\arg\max}\limits_{a_i \in A^N}Q_{i}(z_i', a_i;\bm \theta).
	\end{aligned}
\end{equation}
In order to satisfy the constraints in equation (\ref{loss function_c}), we use the regularization items $(A_{tot}(\bm z',\bm a^*;\bm \theta))^2$ and $(min[A_{tot}(\bm z',\bm a;\bm \theta),0])^2$ to modify the loss function $\mathcal{L}_{td}(\bm \theta)$. The regularization term $(A_{tot}(\bm z',\bm a^*;\bm \theta))^2$ is used to cope with the equality constraint $A_{tot}(\bm z, \bm a^*;\bm \theta)=0$. When $A_{tot}(\bm z, \bm a^*;\bm \theta)\neq0$, the regularization term will give a nonzero penalty value. The regularization term $(min[A_{tot}(\bm z',\bm a;\bm \theta),0])^2$ is used to handle the inequality constraint $A_{tot}(\bm z, \bm a;\bm \theta)\leq0 \quad  if\ \bm a\neq\bm a^*$. When $A_{tot}(\bm z, \bm a;\bm \theta)>0$, the regularization term provides a nonzero penalty term. Through the above two regularization terms, the loss function of policy evaluation will be changed from $\mathcal{L}_{td}(\bm \theta)$ in (\ref{loss function_j}) to
\begin{equation}\label{loss function_t}
	\begin{aligned}
		\mathcal{L}(\bm \theta,\bm v_1,\bm v_2)&=\mathcal{L}_{td}(\bm \theta)+\bm v_1 (A_{tot}(\bm z',\bm a^*;\bm \theta))^2\\
		&+\bm v_2 (min[A_{tot}(\bm z',\bm a;\bm \theta),0])^2,
	\end{aligned}
\end{equation}
where $\bm v_1>0$ and $\bm v_2>0$ are the regularization coefficients. The detailed policy evaluation optimization algorithm is shown in Algorithm \ref{alg1}.
\begin{algorithm}[ht]
	\caption{Policy Evaluation Algorithm}
	\begin{algorithmic}
		\REQUIRE Initial training repaly buffer  $\mathcal{D}$; mixing Q-network and random parameter $\bm \theta$, target mixing Q-network parameter $\bm \theta^-=\bm \theta$, $\epsilon$-greedy policy parameter $\epsilon$, dual variable $\bm v_1$ and $\bm v_2$, initial optimal step size of DNN $\varepsilon$.
		\FOR {each sequence $e=1 \rightarrow E$}
		\STATE Get the initial observation $\bm z_1$ of the environment;
		\FOR {each step $t=1 \rightarrow T$}
		\STATE Take action $a_i^t$ based on the current individual Q-network $Q_i^\theta(z_i^t,a_i^t)$ with $\epsilon$-greedy policy;
		\STATE Get the current joint action $\bm a_{tot}^t=\{a_1^t,a_2^t,...,a_n^t\}$;
		\STATE Perform the current joint action $\bm a_{tot}^t$, get the current reward $r_t$ and get the next moment state observation $\bm z_{t+1}$;
		\STATE Put tuple $\{\bm z_{t},\bm a_{tot}^t,r_t,\bm z_{t+1}\}$ into training replay buffer $\mathcal{D}$;
		\STATE If there is enough data in $\mathcal{D}$, sample minibatch of data $\{\bm z,\bm a_{tot},r,\bm z'\}$ from $\mathcal{D}$;
		\STATE Define $a_i^{max}(z_i';\bm \theta)=\mathop{\arg\max}\limits_{a_i \in A^N}Q_{i}(z_i', a_i;\bm \theta)$;
		\STATE Then obtain the optimal action $\bm a_{max}(\bm z';\bm \theta)=[a_1^{max}(z_1';\bm \theta),a_2^{max}(z_2';\bm \theta),...,a_n^{max}(z_n';\bm \theta)]$;
		\STATE For each of the data of minibatch, use the target mixing Q-network to compute $y=r+\gamma\mathop{\max}\limits_{a' \in A(s')}Q_{tot}(\bm z',\bm a_{max}(\bm z';\bm \theta);\bm \theta^-)$ with the double DQN algorithm;
		\STATE Calculate the TD error loss function $\mathcal{L}_{td}(\bm \theta)=(y-Q_{tot}(\bm z,\bm a;\bm \theta))^2$;
		\STATE Get the regularization terms $A_{tot}(\bm z',\bm a_{max}(\bm z';\bm \theta);\bm \theta)$ and $A_{tot}(\bm z',\bm a;\bm \theta)$;
		\STATE Calculate the global Loss function $\mathcal{L}(\bm \theta,\bm v_1,\bm v_2)=\mathcal{L}_{td}(\bm \theta)+\bm v_1 (A_{tot}(\bm z',\bm a_{max}(\bm z';\bm \theta);\bm \theta))^2+\bm v_2 (min[A_{tot}(\bm z',\bm a;\bm \theta),0])^2$;
		\STATE Calculate the average loss functions of $m$ Samples $\frac{1}{m}\sum_{k=0}^m\mathcal{L}(\bm \theta,\bm v_1,\bm v_2)$;
		\STATE Calculate the gradient of the average loss functions $\nabla_{\bm {\theta}}\frac{1}{m}\sum_{k=0}^m\mathcal{L}(\bm \theta,\bm v_1,\bm v_2)$ for $\bm \theta$;
		\STATE Update parameters $\bm{\theta}\gets\bm \theta-\varepsilon \nabla_{\bm {\theta}}\frac{1}{m}\sum_{k=0}^m\mathcal{L}(\bm \theta,\bm v_1,\bm v_2)$;
		\STATE Update target parameters $\bm \theta^-=\bm \theta$ every $T^-$ time step.
		\ENDFOR
		\ENDFOR
	\end{algorithmic}
	\label{alg1}
\end{algorithm}

\section{Experiments}
\label{Experiments}

In this section, we simulate the nonmonotonic matrix game environment following the approach in Table \ref{tab1}. This simulation aims to demonstrate the effectiveness of the proposed  QFree algorithm in learning non-monotonic joint state-value functions and to show its capability in learning the complete IGM principle. To further assess the adaptability of the QFree algorithm in complex environments, we apply it to the SMAC environment and compare its performance with several classical and advanced MARL algorithms, namely IQL \cite{Tampuu_2017}, VDN \cite{Sunehag2018}, QMIX \cite{Rashid2020}, QTRAN \cite{Son2019}, and QPLEX \cite{Wang2021}. Unlike the matrix game environment, the SMAC environment serves as a classical MARL algorithm validation platform with higher dimensions in state and action spaces, as well as more complex rules. This choice allows for a more comprehensive evaluation of the algorithm's performance and robustness.

\subsection{Matrix Games}
\label{Matrix Games}

In this subsection, we employ a matrix game scenario to validate the effectiveness of the proposed QFree algorithm in non-monotonic cooperative MARL environment. This scenario resembles the one illustrated in Tables \ref{tab1}, where two agents simultaneously choosing action $0$ can attain a maximum reward of $1$. In this environment, we test several classical MARL algorithms based value function factorization for comparison with the QFree algorithm.

\begin{table}
	\caption{Matrix game of different algorithms (red bold is optimal reward).}
	\label{tab2}
	\centering
	\subfloat[IQL: $Q_{tot}$.]{
		\centering
		\begin{tabular}{|c|c|c|c|}
			\hline
			& & & \\[-8pt]
			$\bm a$ & 0 & 1 & 2\\
			\hline
			& & & \\[-8pt]
			0 & -12.0 & -6.1 & -6.0  \\
			\hline 
			& & & \\[-8pt]
			1 & -6.1 & -0.2 & -0.2 \\
			\hline
			& & & \\[-8pt]
			2 & -6.0 & -0.2 &  {\color{red}\textbf{-0.1}} \\
			\hline
		\end{tabular}
		\label{tab2a}
	}
	\subfloat[VDN: $Q_{tot}$.]{
		\centering
		\begin{tabular}{|c|c|c|c|}
			\hline
			& & & \\[-8pt]
			$\bm a$ & 0 & 1 & 2\\
			\hline
			& & & \\[-8pt]
			0 & -23.8  & -11.9 & -11.9  \\
			\hline 
			& & & \\[-8pt]
			1 & -11.8 & 0 & 0 \\
			\hline
			& & & \\[-8pt]
			2 & -11.8 & 0 & {\color{red}\textbf{0}} \\
			\hline
		\end{tabular}
		\label{tab2b}}
		
	\subfloat[QMIX: $Q_{tot}$.]{
		\centering
		\begin{tabular}{|c|c|c|c|}
			\hline
			& & & \\[-8pt]
			$\bm a$ & 0 & 1 & 2\\
			\hline
			& & & \\[-8pt]
			0 & -9.4~ & -9.4& -9.4  \\
			\hline 
			& & & \\[-8pt]
			1 & -9.4~ & 0.0 & 0.0 \\
			\hline
			& & & \\[-8pt]
			2 & -9.4~ & 0.0 & {\color{red}\textbf{0.0}} \\
			\hline
		\end{tabular}
		\label{tab2c}}
	\subfloat[QTRAN: $Q_{tot}$.]{
		\centering
		\begin{tabular}{|c|c|c|c|}
			\hline
			& & & \\[-8pt]
			$\bm a$ & 0 & 1 & 2\\
			\hline
			& & & \\[-8pt]
			0 & {\color{red}\textbf{1.0}} & -12.0 & -12.0  \\
			\hline 
			& & & \\[-8pt]
			1 & -12.0 & 0.0 & 0.0 \\
			\hline
			& & & \\[-8pt]
			2 & -12.0 & 0.0 & 0.0 \\
			\hline
		\end{tabular}
		\label{tab2d}}
		
	\subfloat[QPLEX: $Q_{tot}$.]{
		\centering
		\begin{tabular}{|c|c|c|c|}
			\hline
			& & & \\[-8pt]
			$\bm a$ & 0 & 1 & 2\\
			\hline
			& & & \\[-8pt]
			0 & -0.2 & -11 & -11 \\
			\hline 
			& & & \\[-8pt]
			1 & -11.0 & {\color{red}\textbf{-0.2}} & -0.2 \\
			\hline
			& & & \\[-8pt]
			2 & -11.4 & -0.2 & -0.2 \\
			\hline
		\end{tabular}
		\label{tab2e}}
	\subfloat[QFree: $Q_{tot}$.]{
		\centering
		\begin{tabular}{|c|c|c|c|}
			\hline
			& & & \\[-8pt]
			$\bm a$ & 0 & 1 & 2\\
			\hline
			& & & \\[-8pt]
			0 & {\color{red}\textbf{1.0}} & -12.0 & -12.0  \\
			\hline 
			& & & \\[-8pt]
			1 & -12.0 & 0.0 & 0.0 \\
			\hline
			& & & \\[-8pt] 
			2 & -12.0 & 0.0 & 0.0 \\
			\hline
		\end{tabular}
		\label{tab2f}}
\end{table}

\begin{figure}[!t]
	\centering
	\includegraphics[width=3in]{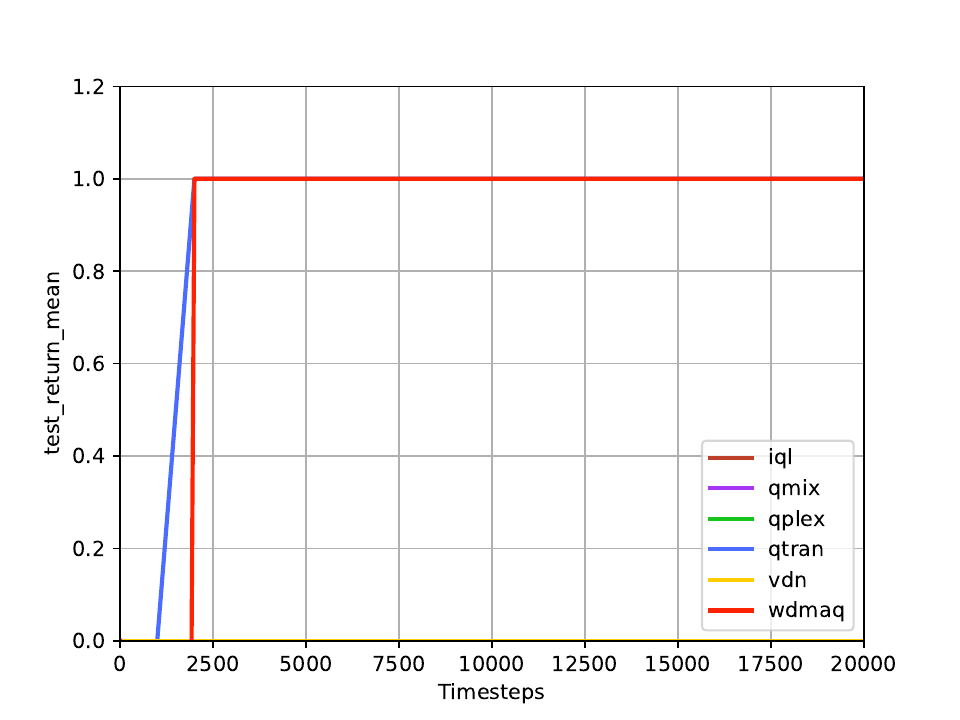}
	\caption{Mean return of matrix game  for different algorithms.}
	\label{Matrix game}
\end{figure}

The experimental results are presented in Table \ref{tab2}, which show that apart from the proposed QFree algorithm and QTRAN, none of the other algorithms are able to learn the optimal policy. IQL trains each agent independently, transforming MARL into single-agent reinforcement learning, which is unable to learn the optimal joint action value function. VDN and QMIX introduce monotonic constraints to satisfy the IGM principle, but they fail to learn the optimal policy in a non-monotonic environment. The QPLEX algorithm transforms monotonic constraints to the advantage function but does not completely eliminate the constraints. So it is still incapable of learning the optimal policy. To further confirm the aforementioned analysis, the relationships between the average test return and the number of training steps are provided in Fig.\ref{Matrix game}. It can be observed that only QTRAN and the QFree algorithm achieve the optimal return of $1$. This demonstrates that the QFree algorithm can learn the optimal policy in non-monotonic environments.

To further validate the effectiveness of the algorithm, we expand their action space in order to reflect the training effect. In this setting, we have two agents, and each agent has $21$-dimensional action space $A_i=\{0,1,...,19,20\}$. The entire matrix game environment reward function is designed as follows\cite{Son2019}:
\begin{equation}\label{matrix}
 \begin{aligned}
  f_1(a_1, a_2)=5-(\frac{15-a_1}{3})^2-(\frac{5-a_2}{3})^2\\
  f_2(a_1, a_2)=10-(\frac{5-a_1}{1})^2-(\frac{15-a_2}{1})^2\\
  R(a_1, a_2)=\max(f_1(a_1, a_2), f_2(a_1, a_2)).
 \end{aligned}
\end{equation}
From (\ref{matrix}), we can get the relationship between its reward function and the two agents' actions $a_1$ and $a_2$, as shown in Fig. \ref{Matrix game1}.

\begin{figure}[!t]
 \centering
 \includegraphics[width=3.5in]{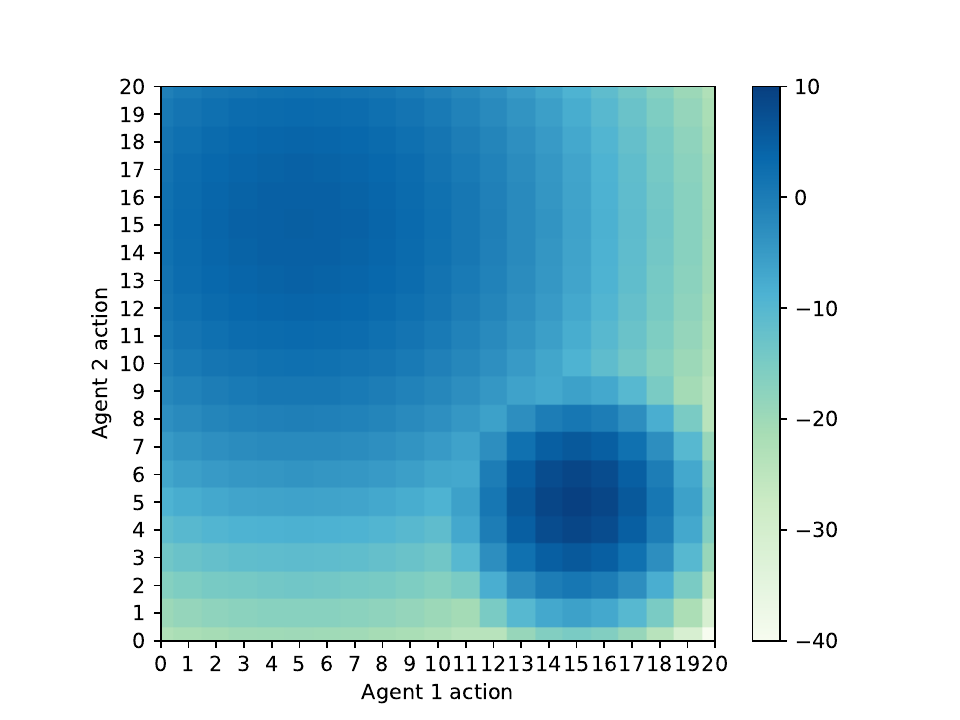}
 \caption{Matrix game reward function distribution.}
 \label{Matrix game1}
\end{figure}

\begin{figure}[!t]
	\centering
	\subfloat[step 0.]{
		\includegraphics[width=0.5\linewidth]{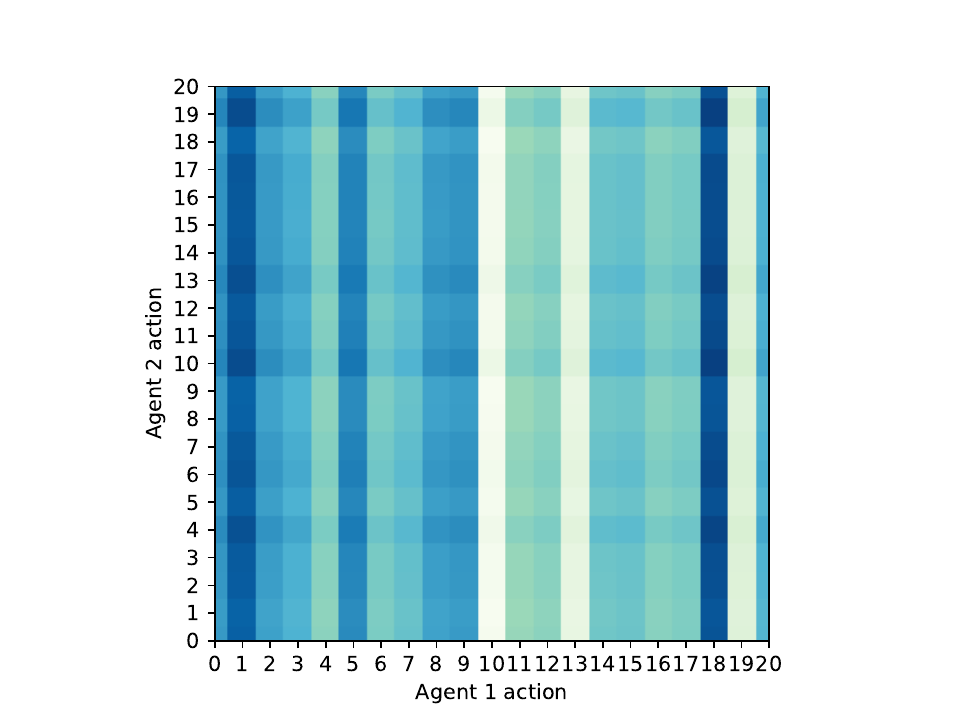}	
	\label{0}}
	\subfloat[step 500.]{
	\includegraphics[width=0.5\linewidth]{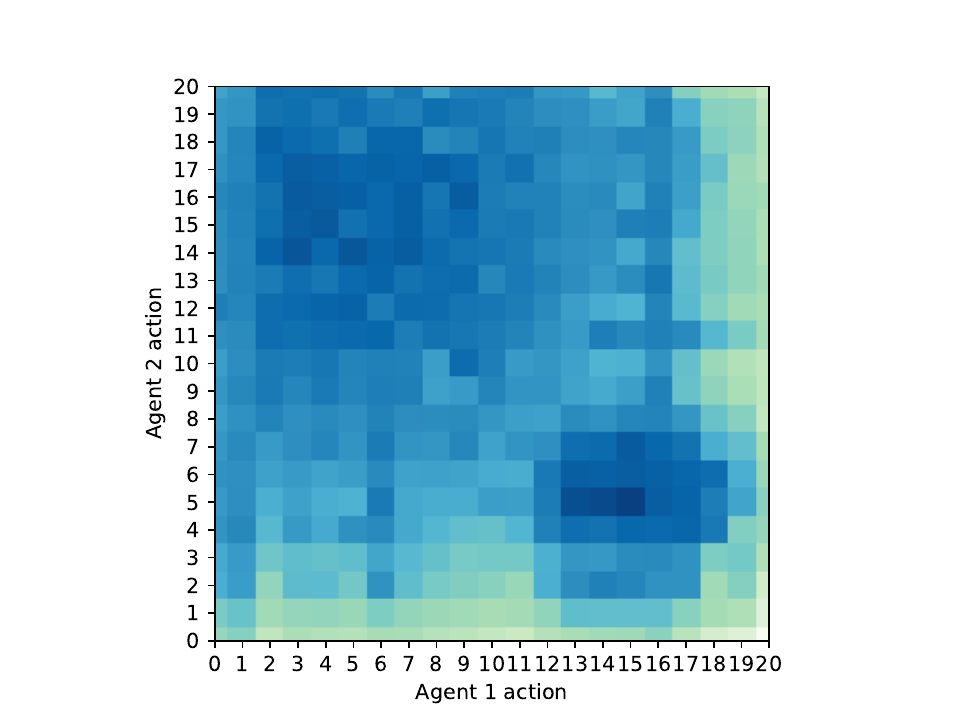}	
	\label{500}}
	
	\subfloat[step 1000.]{
	\includegraphics[width=0.5\linewidth]{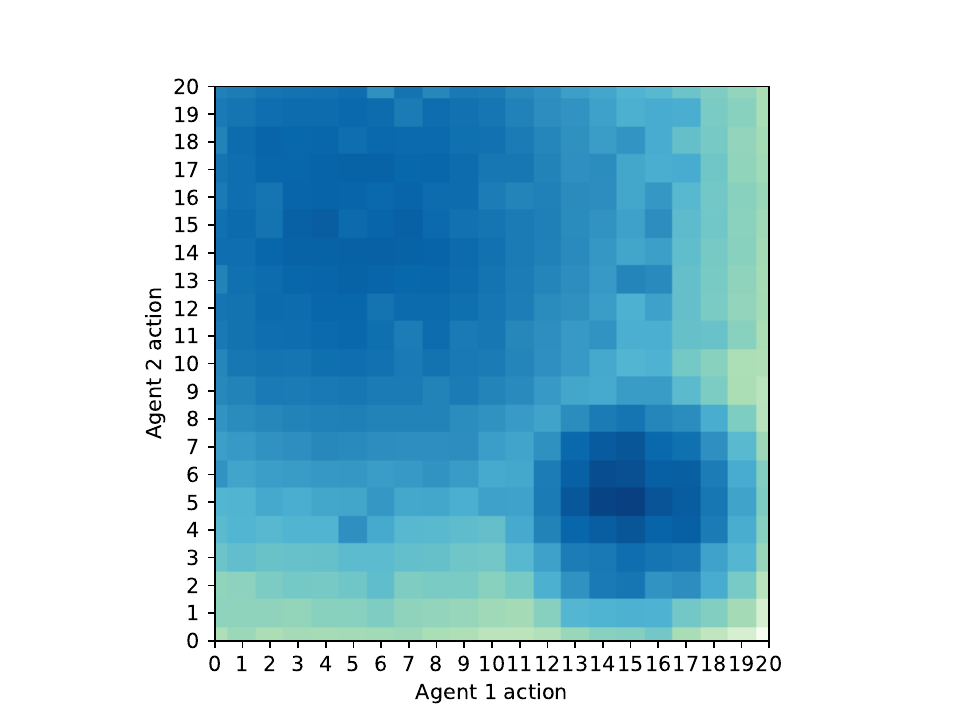}	
	\label{1000}}
	\subfloat[step 2000.]{
	\includegraphics[width=0.5\linewidth]{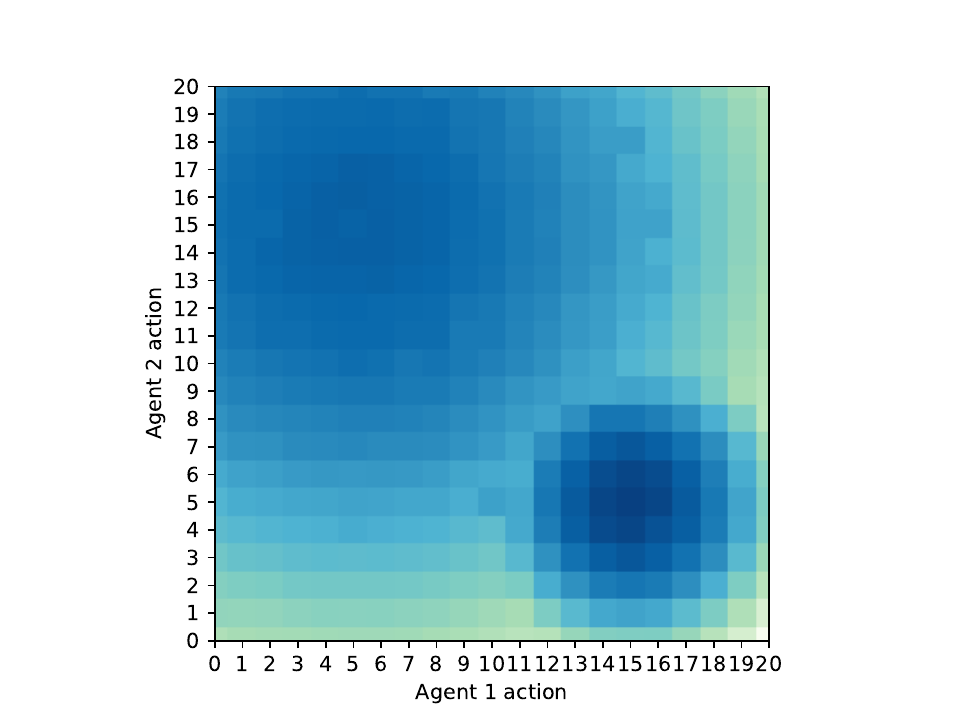}	
	\label{2000}}
	\caption{Matrix game training process.}
	\label{Matrix game2}
\end{figure}

In Fig. \ref{Matrix game1}, it can be seen that when the joint action $a_{tot}=(a_1,a_2)$ is taken as $a_{tot}=(5,15)$, the whole reward is taken to the maximum. Additionally, there also exists a locally optimal solution $a_{tot}=(15,5)$. As a result, the relationship between the global action value function and the individual action value function is non-monotonic.

The training process of the proposed algorithm is presented in Fig. \ref{Matrix game2}. We can see that after 2000 steps of learning, the algorithm has successfully learned the reward function matrix. The results clearly demonstrate the ability of the proposed algorithm to handle more complex MARL problems. 

\subsection{SMAC}
\label{SMAC}

\begin{table*}[!t]
	\begin{center}
		\caption{Winning rates of different algorithms on the SMAC map.}\label{tab4}
		\begin{tabular}{lccccccc} 
			\toprule[0.5mm]
			\textbf{Maps} & \textbf{Difficulty} & \textbf{IQL} & \textbf{VDN} & \textbf{QMIX}  & \textbf{QTRAN} & \textbf{QPLEX} & {\color{red} \textbf{QFree}}\\
			\midrule[0.3mm]
			2s3z & Easy & \textbf{\color{blue} 72.2\%} & 95.9\% & 96.8\% & 90.2\%  & 98.8\% & \textbf{\color{red} 99.5\%}\\
			2s\_vs\_1sc & Easy & \textbf{\color{blue} 96.9}\% & 99.4\% & 99.2\% & 99.2\%  & 98.9\% & \textbf{\color{red}100\%}\\
			3s\_vs\_4z & Easy & 84.5\% & 97.2\% & 98.4\% & \textbf{\color{blue} 20.9\%}  & 99.1\% & \textbf{\color{red}99.4\%}\\
			MMM & Easy & 88.1\% & 97.8\% & 98.4\% & \textbf{\color{blue}85.8\%}  & 99.5\% & \textbf{\color{red}100\%}\\
			so\_many\_baneling & Easy & \textbf{\color{blue} 56.3\%} & 95.2\% & \textbf{\color{red}98.0\%} & 92.8\%  & 96.3\% &  95.0\%\\
			\bottomrule[0.3mm]
			5m\_vs\_6m & Hard & \textbf{\color{blue} 44.2\%} & 67.3\% & 62.8\% & 55.8\%  & 73.4\% & \textbf{\color{red} 75.9\%}\\
			8m\_vs\_9m & Hard & \textbf{\color{blue} 30.5\%} & 88.4\% &  \textbf{\color{red}92.2\%}  & 65.3 \% &74.1\% & 85.0\%\\
			10m\_vs\_11m & Hard & \textbf{\color{blue} 29.1\%} & 90.6\% & \textbf{\color{red} 96.1\%} & 70.8\%  & 80.6\% & 87.5\%\\
			1c3s5z & Hard & \textbf{\color{blue} 14.8\%} & 88.9\% & 93.8\% & 48.1\%  &  95.0\% & \textbf{\color{red}97.2\%}\\
			3s\_vs\_5z & Hard & 39.1\% & 88.0\% & 84.4\% & \textbf{\color{blue} 3.6\%}  &  93.4\% & \textbf{\color{red} 95.3\%}\\
			\midrule[0.3mm]
			3s5z & S-Hard & \textbf{\color{blue} 8.1\%} & 67.7\% & 94.7\% & 14.7\%  & \textbf{\color{red}95.8}\% & 94.7\%\\
			25m & S-Hard & \textbf{\color{blue} 11.9\%} & 92.5\% & 98.1\% & 62.2\%  & 92.2\% &  \textbf{\color{red}99.4\%}\\
			27m\_vs\_30m & S-Hard & \textbf{\color{blue} 0.0\%} & 8.1\% &  26.9\% & 0.9\%  & 18.1\% & \textbf{\color{red}31.3\%}\\
			3s5z\_vs\_3s6z &S-Hard & \textbf{\color{blue} 0.0\%} & 0.3\% & 0.9\% & \textbf{\color{blue} 0.0\%}  &  7.8\% & \textbf{\color{red}13.5\%}\\
			MMM2 & S-Hard & \textbf{\color{blue} 0.3\%} & 12.7\% &  70.2\% &  0.8\% & 17.3\% & \textbf{\color{red}77.8}\%\\
			6h\_vs\_8z & S-Hard & \textbf{\color{blue} 0.0\%} &  4.7\% & 2.0\% & 0.6\%  & 3.8\% & \textbf{\color{red}5.3\%}\\
			2c\_vs\_64zg & S-Hard & \textbf{\color{blue} 27.2\%} & 31.4\% & 52.9\% & 42.2\%  & 34.5\% & \textbf{\color{red} 66.3\%}\\
			corridor & S-Hard & 0.2\% & 1.5\% & \textbf{\color{blue} 0.0\%} & 0.3\%  &  2.0\% & \textbf{\color{red}38.8}\%\\
			\bottomrule[0.5mm]
		\end{tabular}
	\end{center}
\end{table*}

\begin{figure*}[!t]
	\centering
	\subfloat[1c3s5z.]{
		\includegraphics[width=0.32\linewidth]{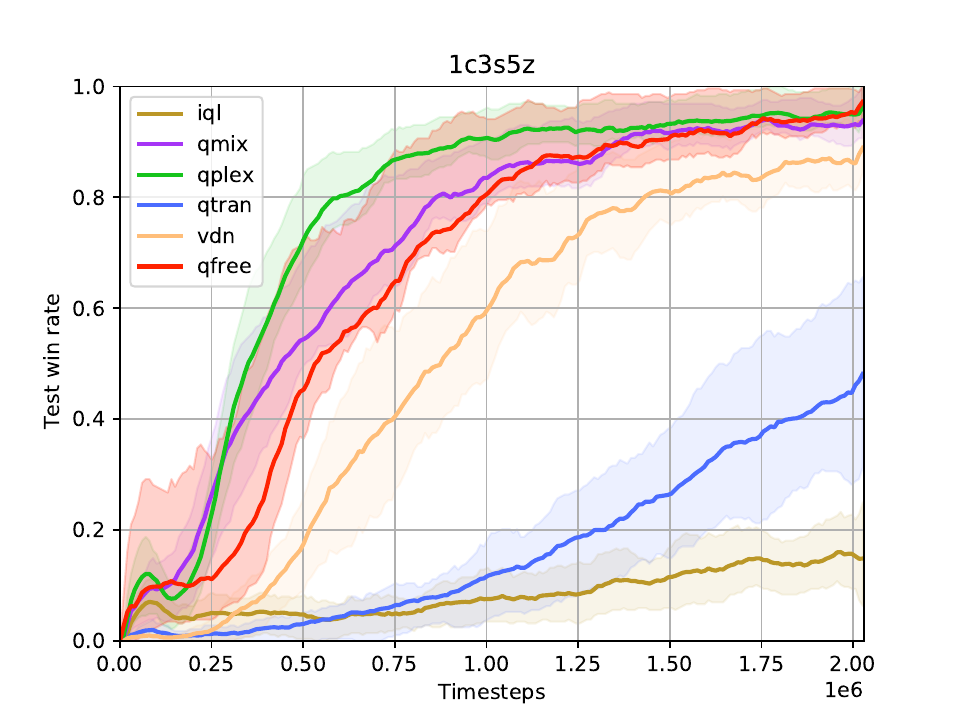}
		\label{1c3s5z}}
	\hfill 
	\subfloat[3s5z.]{
		\includegraphics[width=0.315\linewidth]{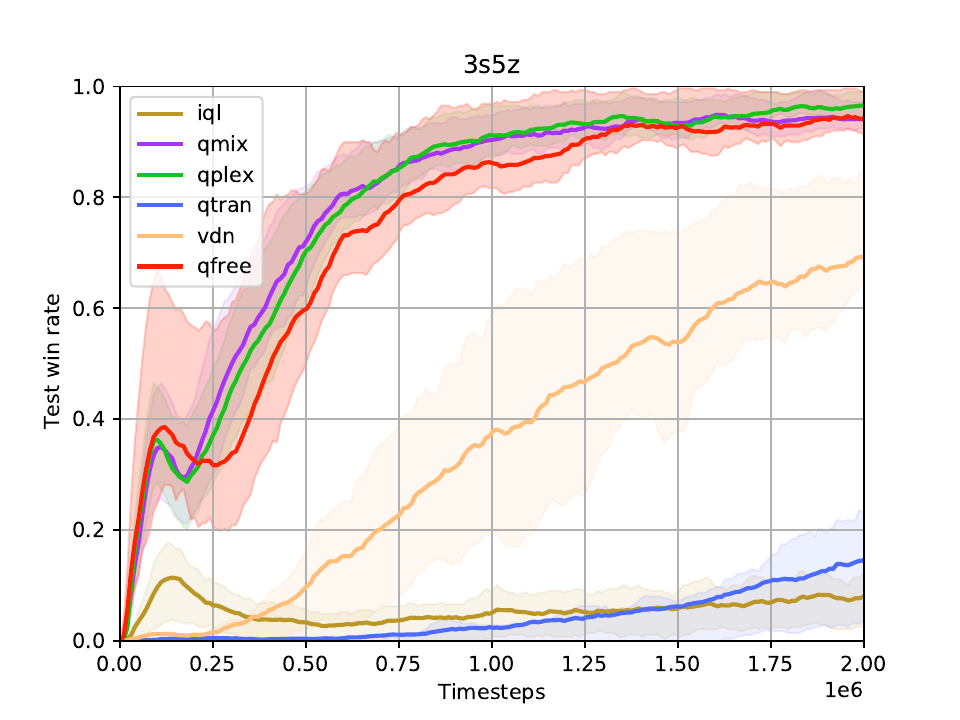}	
	\label{3s5z}}
	\hfill 
	\subfloat[3s\_vs\_5z.]{
		\includegraphics[width=0.315\linewidth]{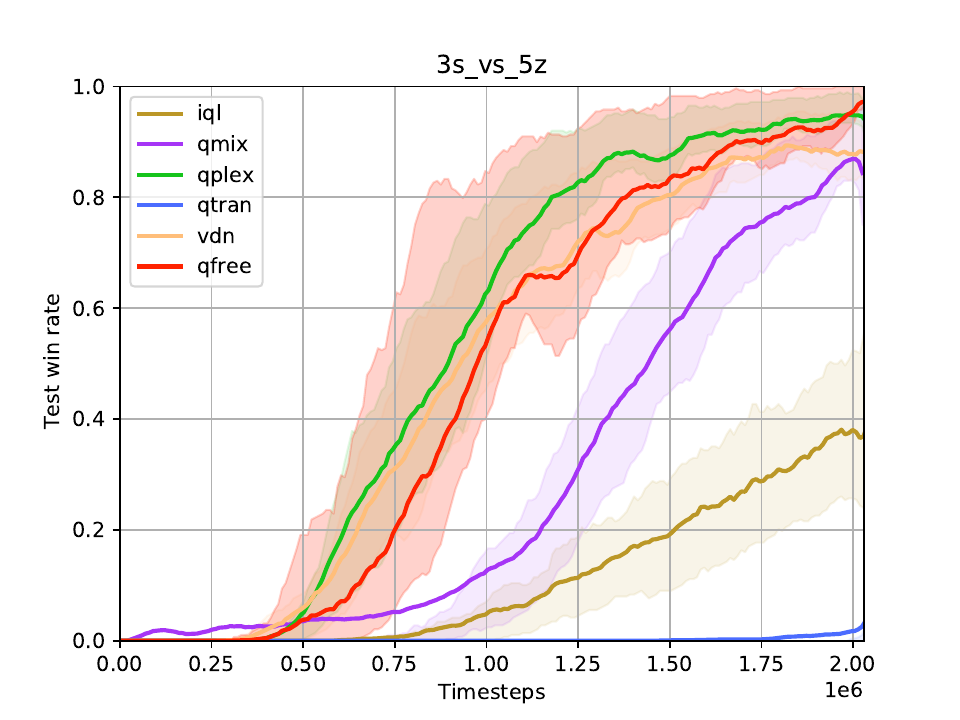}	
	\label{3svs5z}}
	\hfill 
	\subfloat[5m\_vs\_6m.]{
		\includegraphics[width=0.315\linewidth]{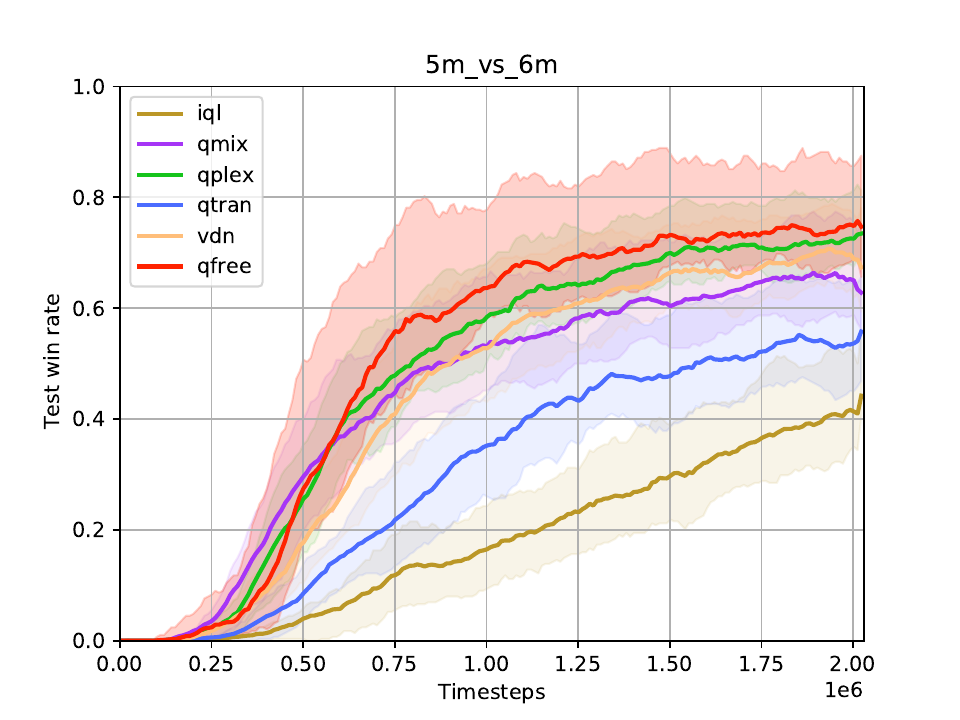}	
	\label{5mvs6m}}
	\hfill
	\subfloat[2c\_vs\_64zg.]{
		\includegraphics[width=0.315\linewidth]{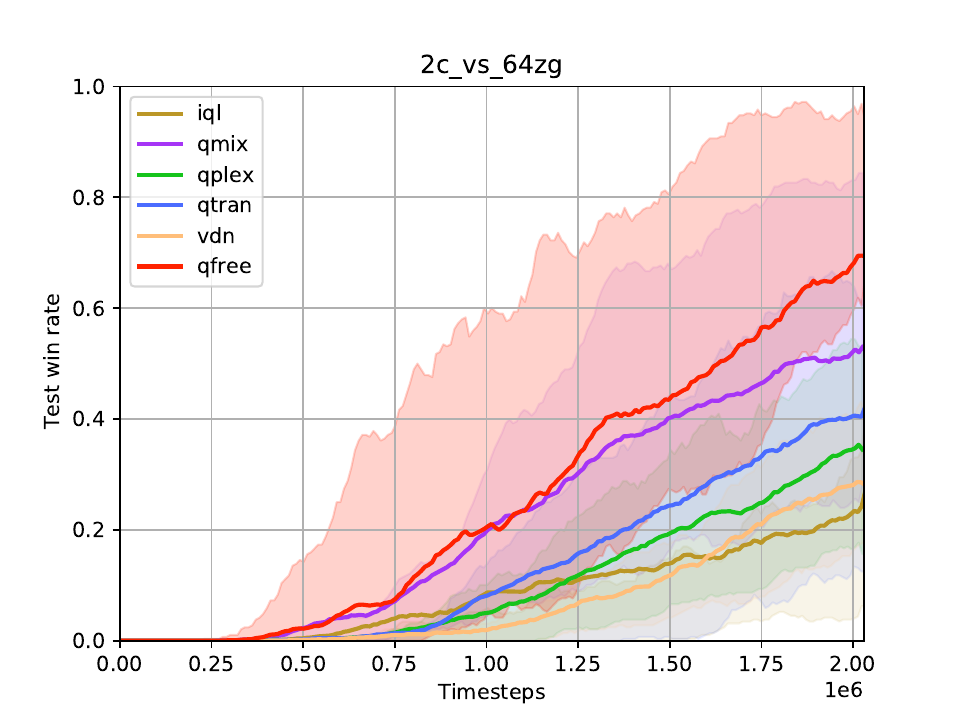}	
	\label{2cvs64zg}}
	\hfill
	\subfloat[MMM2]{
		\includegraphics[width=0.315\linewidth]{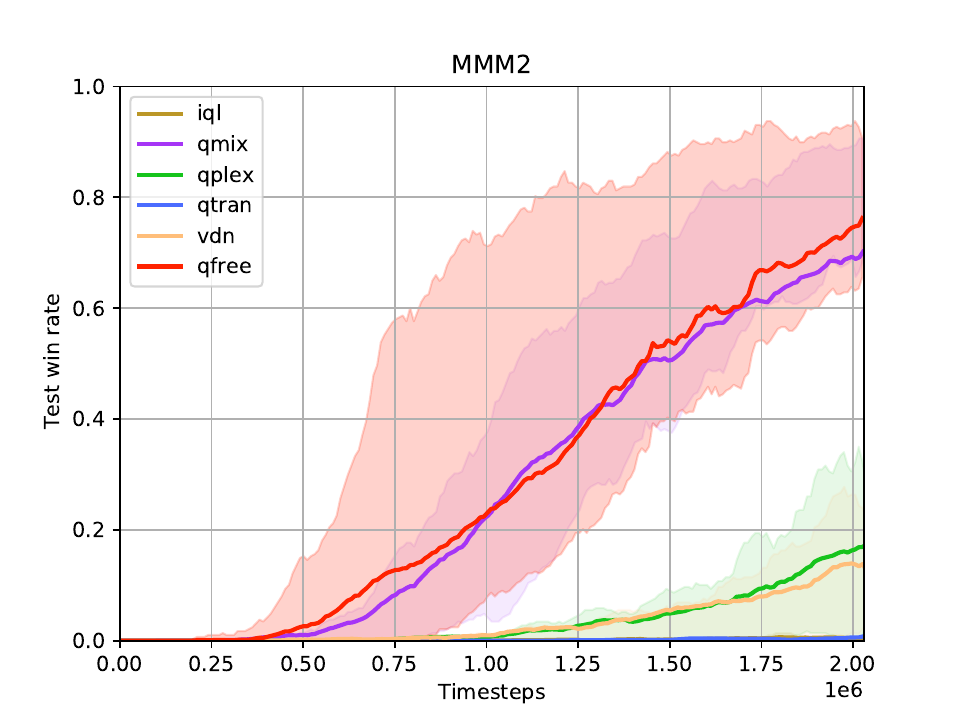}	
	\label{MMM2}}
	\hfill 
	\subfloat[3s5z\_vs\_3s6z.]{
		\includegraphics[width=0.315\linewidth]{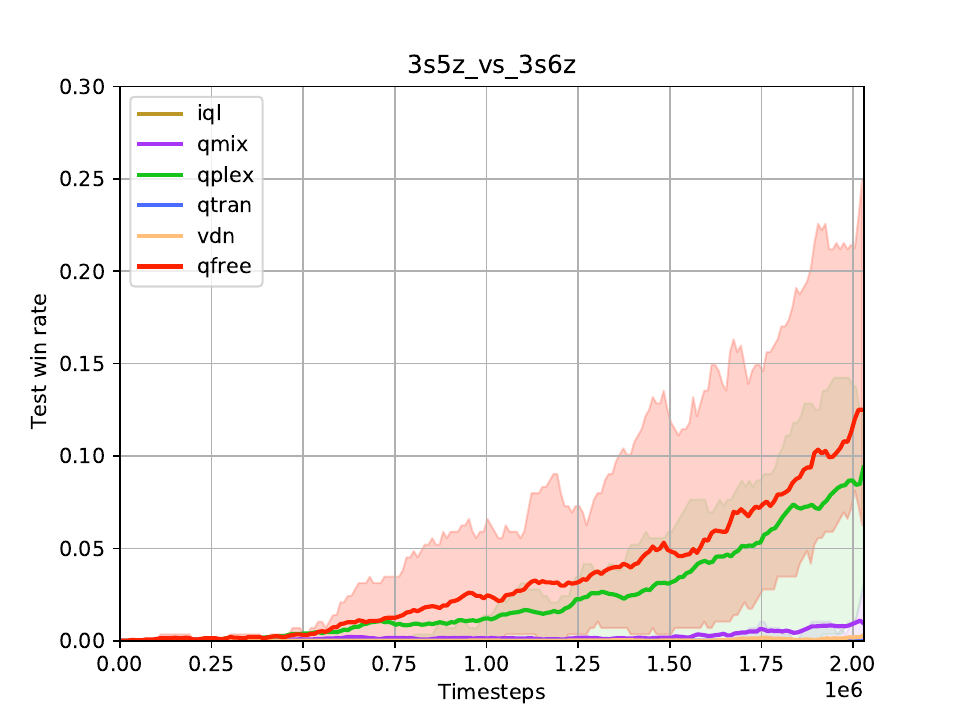}	
	\label{3s5z_vs_3s6z}}
		\hfill 
	\subfloat[Corridor.]{
		\includegraphics[width=0.315\linewidth]{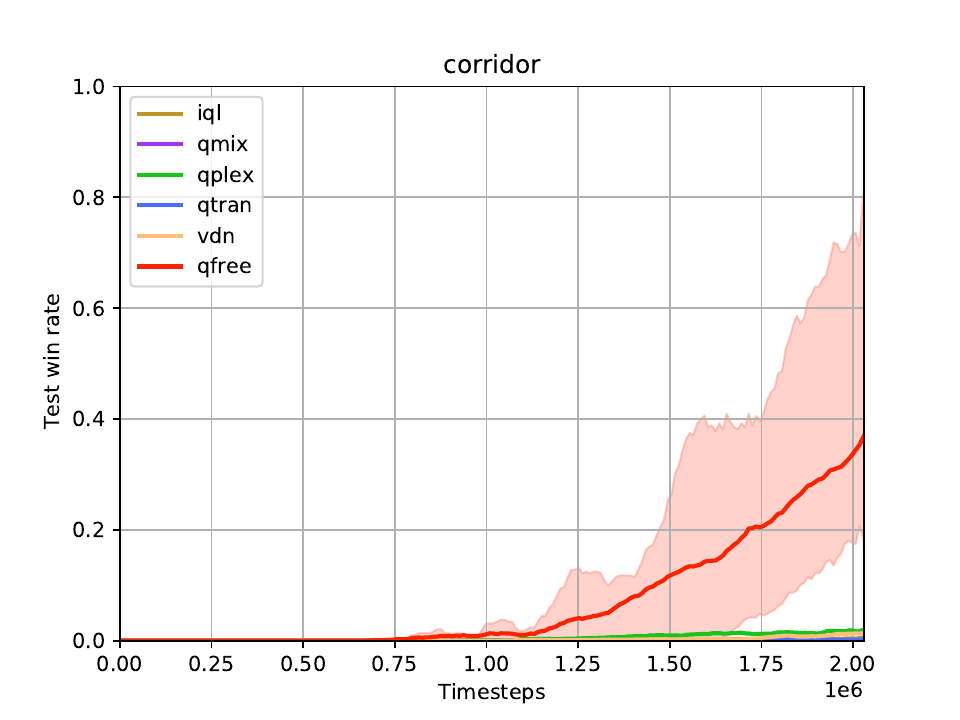}	
		\label{corridor}}
	\hfill 
	\subfloat[6h\_vs\_8z.]{
		\includegraphics[width=0.315\linewidth]{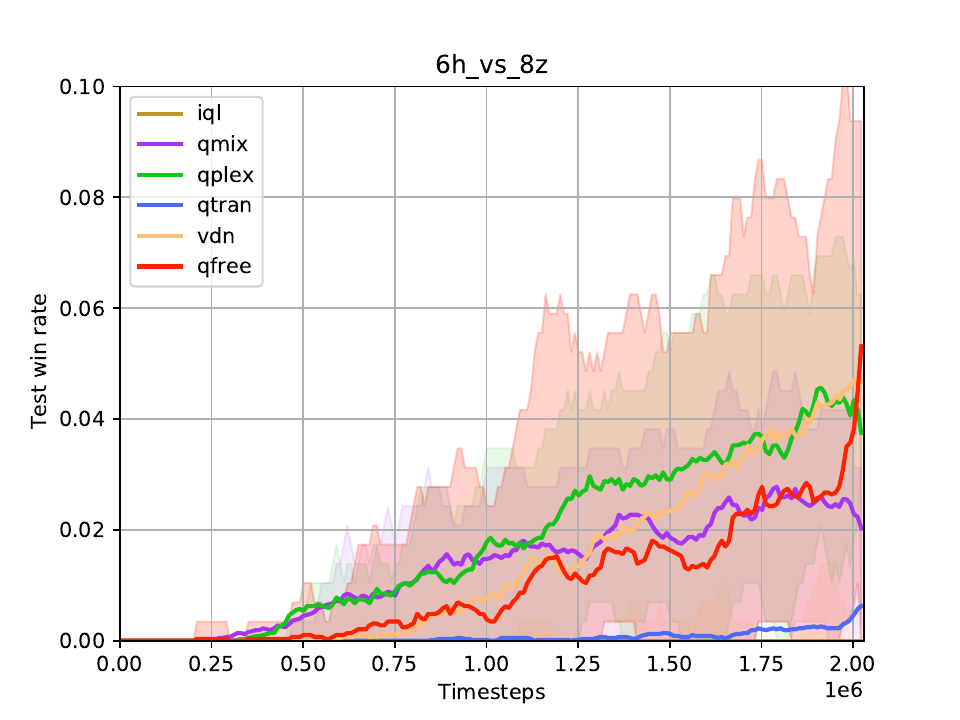}	
	\label{6hvs8z}}

	\caption{The test winning rate of different algorithms.}
	\label{fig_SMAC}
\end{figure*}
	
SMAC is a generalized MARL algorithm testing environment, which are more complicated than the matrix games, thus providing more convinced results to test algorithm performance. SMAC is based on Starcraft II, a popular real-time strategy (RTS) game. Similar to most RTS games, Starcraft II is divided into macromanagement (economy and resource management) and micromanagement (combat unit control). SMAC focuses on micromanagement, which aims to train ally units to defeat the Starcraft II game's built-in scripted AI-controlled enemy units through an MARL algorithm. 
	
In Starcraft II, every unit has the ability to perform independent actions such as moving, attacking enemy units and healing ally units. Different maps in the game feature various types of units, including Stalkers, Zealots, Colossi, Marines, Banelings and Medivacs, and each with their own distinct abilities like healing, self-destructing, shielding and so on. The main objective of the game is to maximize damage inflicted on enemy units while minimizing damage to ally units through careful micromanagement. To achieve this objective, ally agents often need to learn and follow specific policies, such as focusing fire to prioritize the destruction of specific enemies or kiting to utilize range advantage and control distances effectively. The evaluation of algorithm performances is facilitated by the simulation of SMAC, which offers a variety of maps of different difficulty levels as shown in Table \ref{tab4}. These maps are categorized into three levels: easy, hard, and super hard, with the latter being extremely challenging for algorithms. Each map presents unique battlefield environments and includes different types and quantities of ally and enemy units.

In the SMAC environment, the state of the agent includes the position, health, shield, and unit type of all the ally and enemy units, as well as map information. However, each unit can only observe a limited area within its sight range. The observation includes the distance, relative position $x$, relative position $y$, health, shield, and unit type of ally and enemy units within this sight range. Outside of this observation range, the unit has no knowledge of other unit information. This setup creates a Dec-POMDP environment. In each step, ally unit agent takes actions from a discrete action space based on its observation information. The action space includes options such as ``no operation", ``move [direction]", ``attack [enemy unit id]", ``heal [ally unit id]" (for maps with Medivacs), and ``stop". The size of the action space depends on the number of enemy units present on the map. The reward function is consistent across the maps, with killing enemy units resulting in a reward of $10$ and winning the game resulting in a larger reward of $200$.

In order to assess the superiority of the proposed QFree algorithm, the performance of this algorithm is compared with popular MARL algorithms, namely IQL, VDN, QMIX, QTRAN and QPLEX, in the SMAC environment. To ensure a fair comparison, the reward function and other training hyperparameters are kept the same for all algorithms. Table \ref{tab4} presents the winning rates of the algorithms after training for $2.03$ million time steps on different maps. The results are averaged over 20 training sessions to reduce the randomness. The highest winning rate is indicated in bold red, while the lowest is indicated in bold blue. It can be observed that the QFree algorithm achieves the highest winning rate on most maps. Moreover, QTRAN performs well in matrix games, but shows lower winning rates than most methods in complex environments. Fig. \ref{fig_SMAC} illustrates the growth curves of the win rates for the different algorithms on $9$ representative maps in SMAC, where the curves represent the average win rate over the $20$ runs, with the shaded area indicating the $75\%$ confidence interval. From the figure, it can be seen that the proposed QFree algorithm achieves the highest win rate on challenging SMAC maps in most maps.

Overall, the comparison results demonstrate the superiority of the QFree algorithm in the SMAC environment compared to other MARL algorithms. The QFree algorithm achieves higher win rates, making it a promising choice for addressing challenges in complex environments.

Moreover, the training results of the QFree algorithm applied to the MMM2 and 2c\_vs\_64zg maps are showcased. Representative video frames are provided in Figs. \ref{MMM2c} and \ref{2cvs64zgc}. The attainment of the focusing fire skill by the SMAC agent is demonstrates in Figs. \ref{focusing fire_1} and \ref{focusing fire_2}. Specifically, the attack unit in the MMM2 map identifies a priority to focus fire on the healing unit. Similarly, in the 2c\_vs\_64zg map, two Colossi are observed to prioritize focusing fire on one side enemy unit. Fig.\ref{healing} reveals that the Medivac tends to prioritize the healing of frontline ally units that have taken damage. Fig.\ref{kiting} evidence that the Colossi are capable of learning the tactic of kiting. This strategy involves protecting the unit from harm by maintaining a safe distance from enemy units and leveraging the range advantage to assault the enemy. In conclusion, under the instruction of the QFree algorithm, the agent is capable of acquiring and implementing certain human-like game policies in Starcraft II, ultimately securing a victory as illustrated in Figs. \ref{win_1} and \ref{win_2}.

\begin{figure*}[!t]
	\centering
	\subfloat[Focusing fire.]{
		\includegraphics[width=0.255\linewidth]{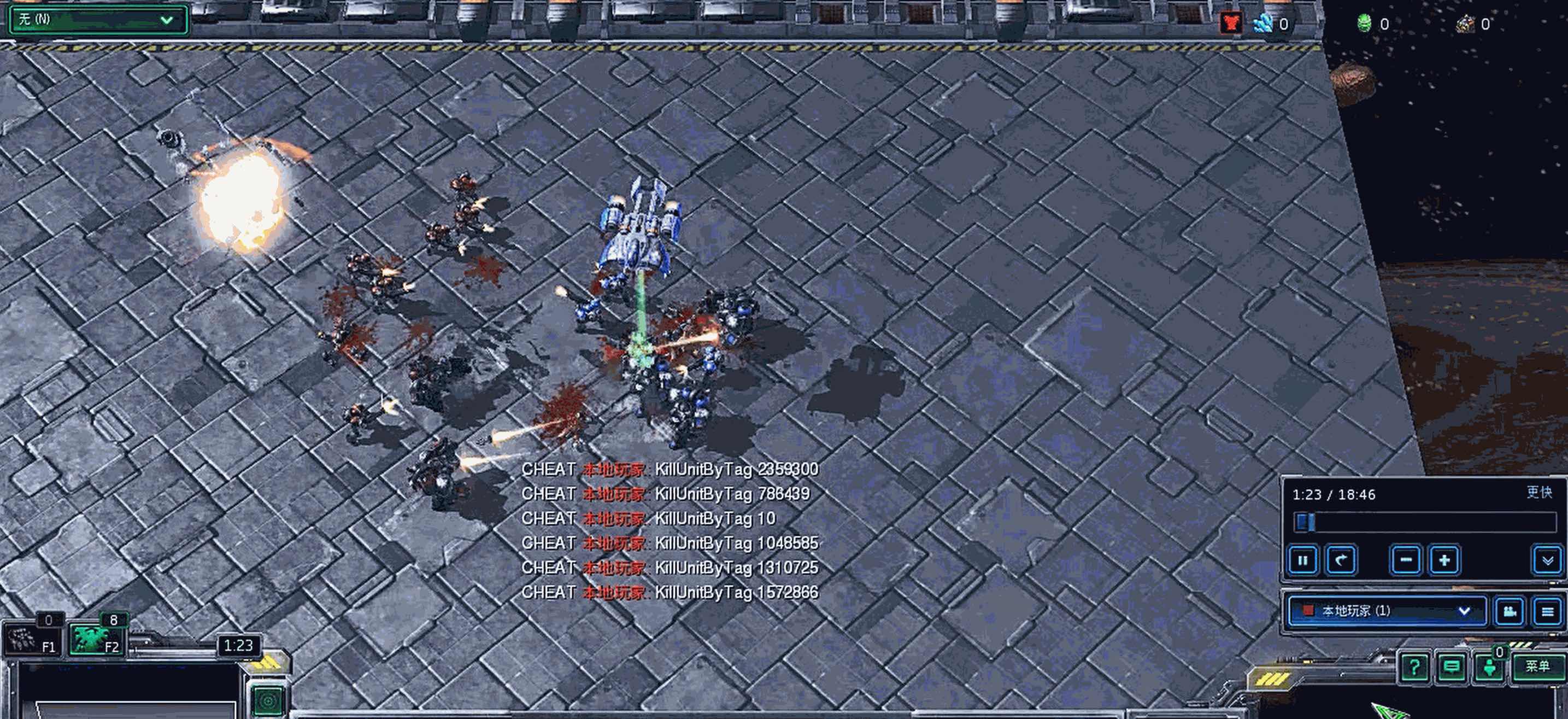}	
		\label{focusing fire_1}}
	\hfill  
	\subfloat[Healing.]{
		\includegraphics[width=0.255\linewidth]{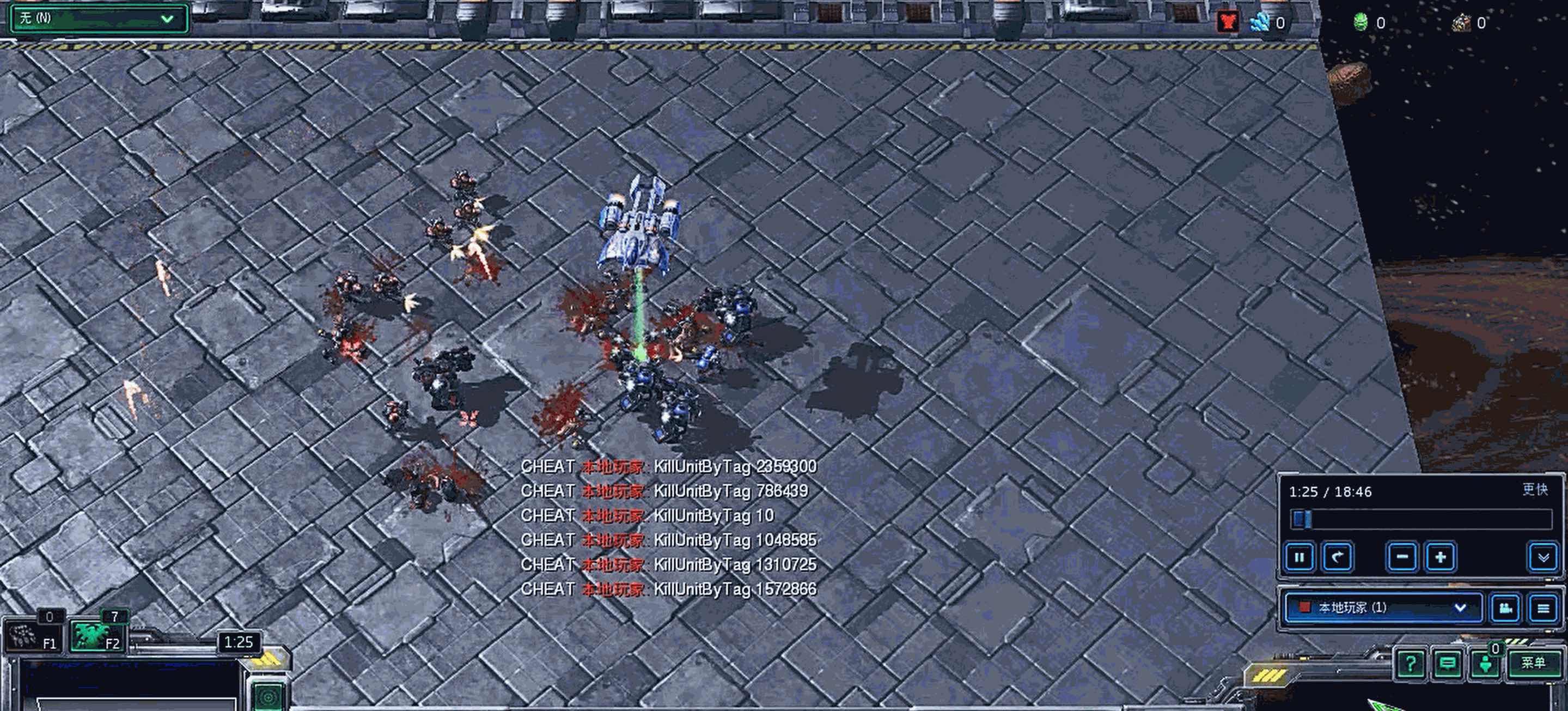}	
		\label{healing}}
	\hfill 
	\subfloat[Win.]{
		\includegraphics[width=0.255\linewidth]{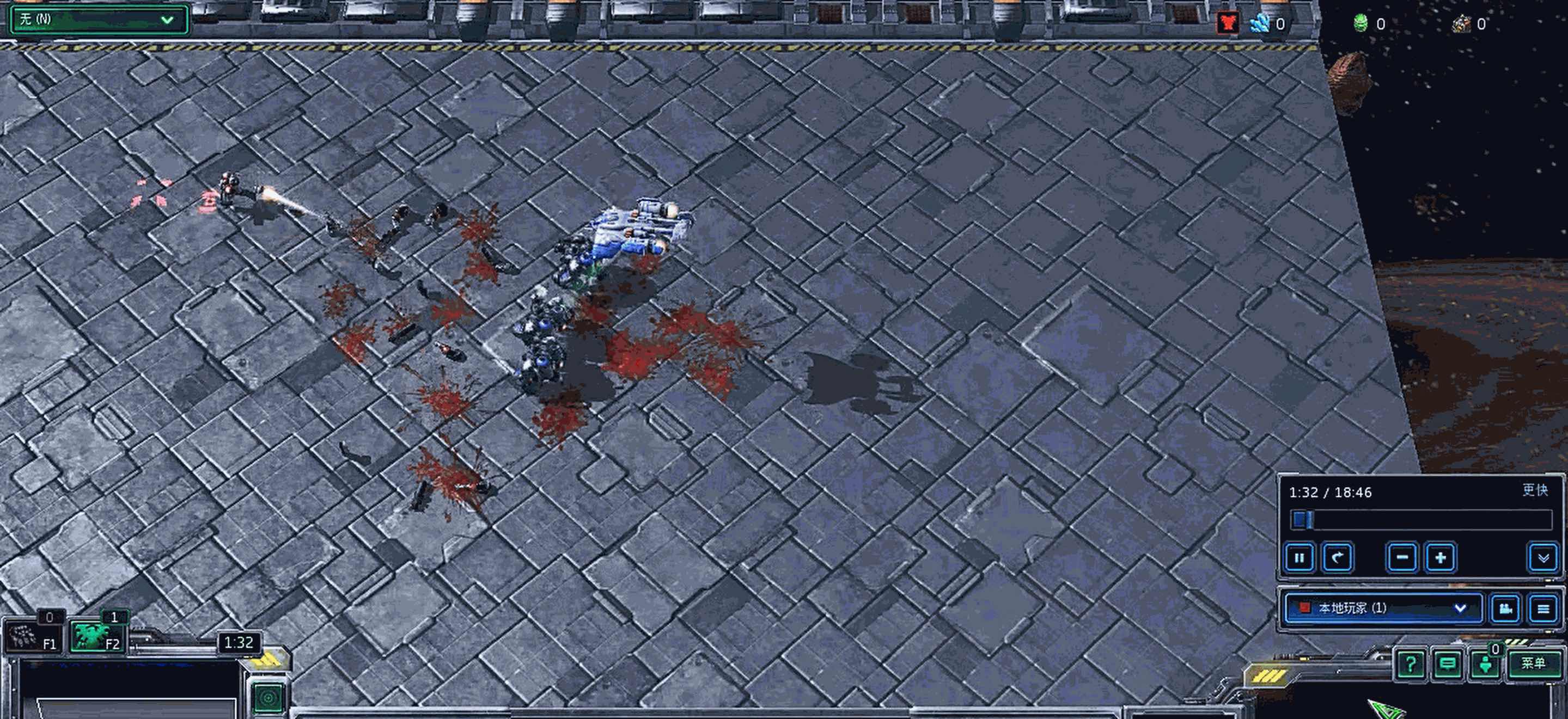}	
		\label{win_1}}
	\caption{MMM2 battle operation flowchart.}
	\label{MMM2c}
\end{figure*}

\begin{figure*}[!t]
	\centering
	\subfloat[Focusing fire.]{
		\includegraphics[width=0.255\linewidth]{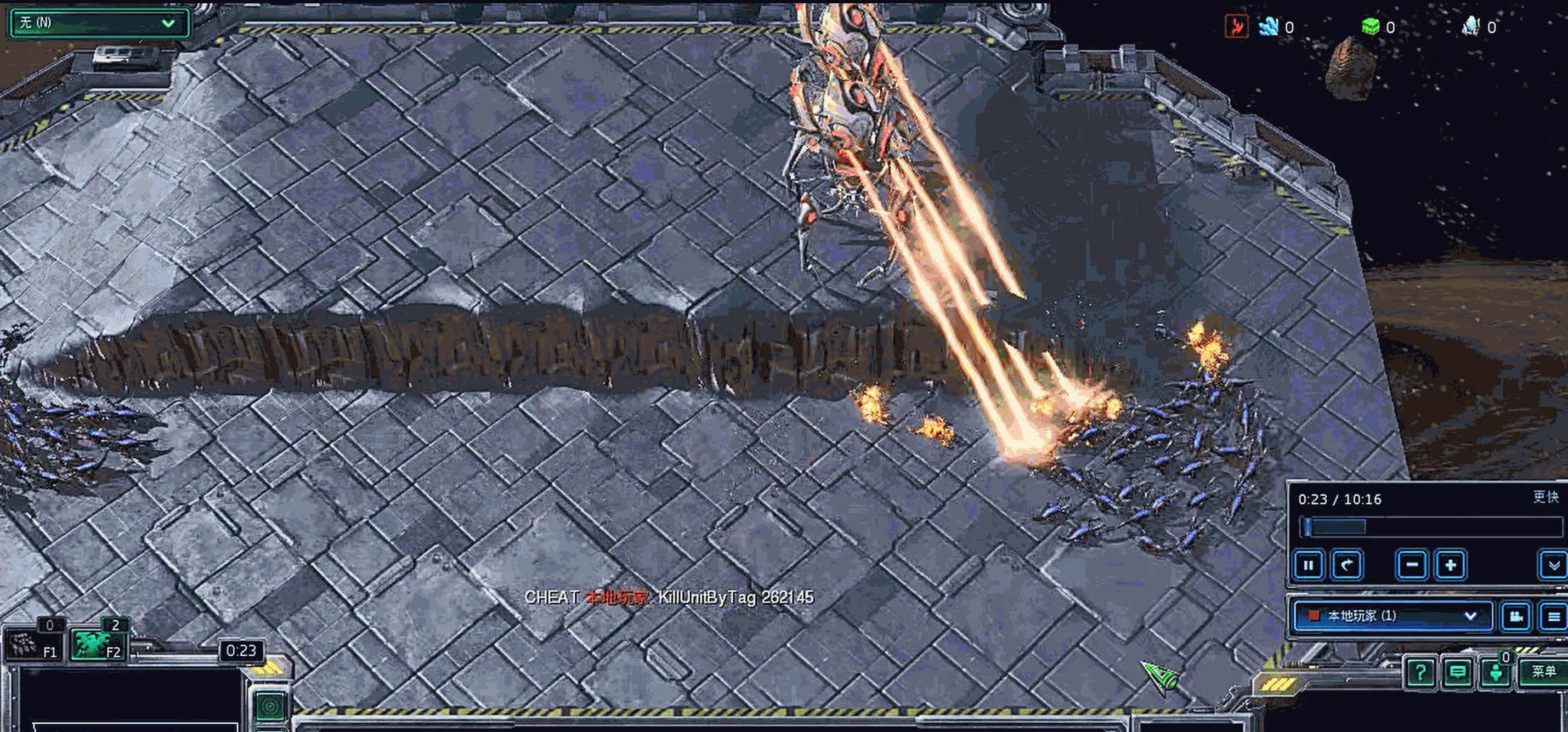}	
		\label{focusing fire_2}}
	\hfill  
	\subfloat[Kiting.]{
		\includegraphics[width=0.255\linewidth]{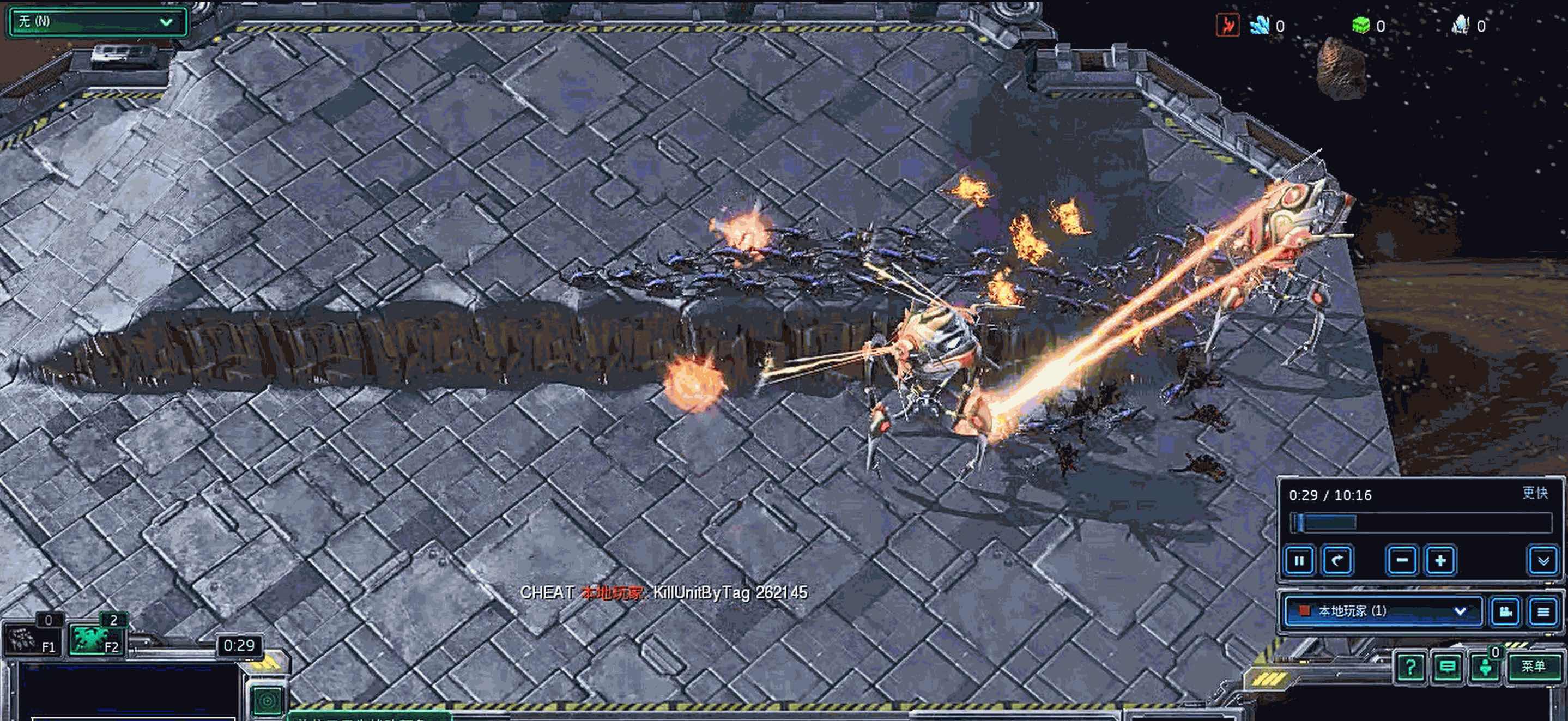}	
		\label{kiting}}
	\hfill 
	\subfloat[Win.]{
		\includegraphics[width=0.255\linewidth]{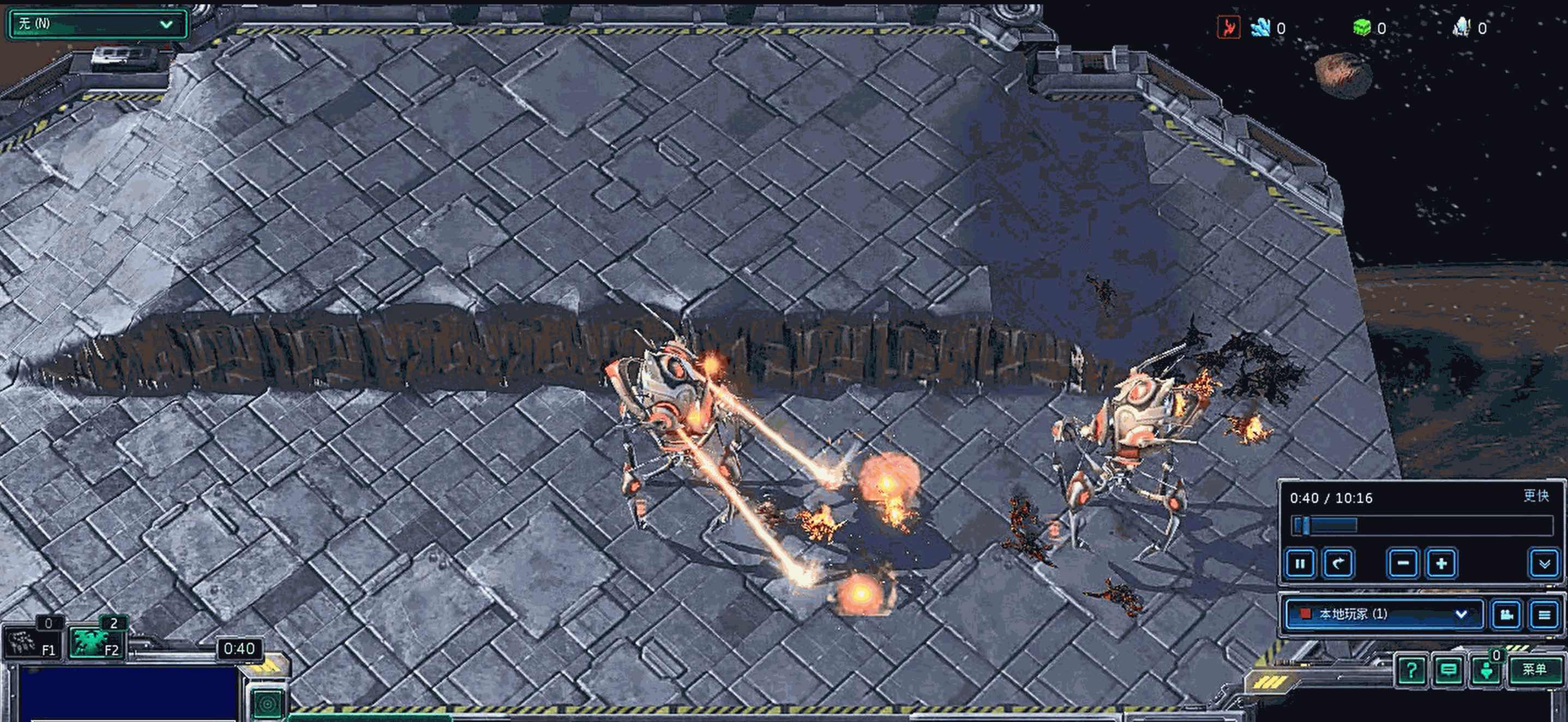}	
		\label{win_2}}
	\caption{2c\_vs\_64zg battle operation flowchart.}
	\label{2cvs64zgc}
\end{figure*}

\subsection{Ablation Experiments}
\label{Ablation Experiments}

In this section, we employ two ablation experiments to validate the effectiveness of the proposed mixing network structure and the MARL algorithm. Firstly, we conduct an ablation experiment to confirm that the designed mixing network structure improves the expression ability of the non-monotone joint action value function. Subsequently, we carry out another ablation experiment to validate the necessity of the regularization term in the policy evaluation algorithm.

\subsubsection{Mixing Network Architecture}

\begin{figure*}[!t]
	\centering
	\subfloat[step 0.]{
		\includegraphics[width=0.25\linewidth]{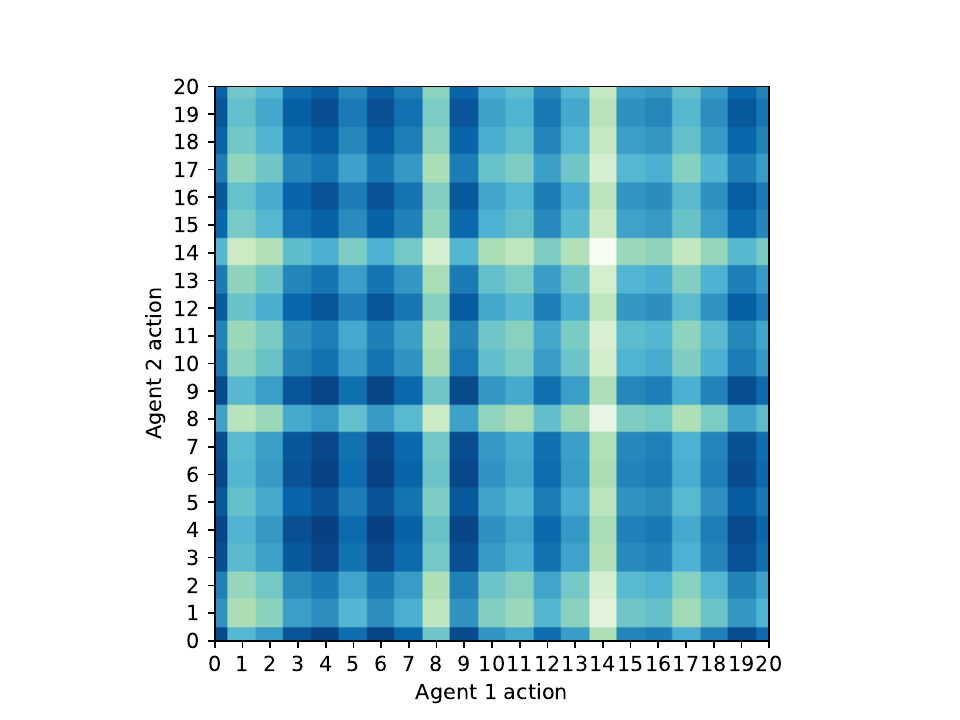}	
		\label{a0}}
	\subfloat[step 500.]{
		\includegraphics[width=0.25\linewidth]{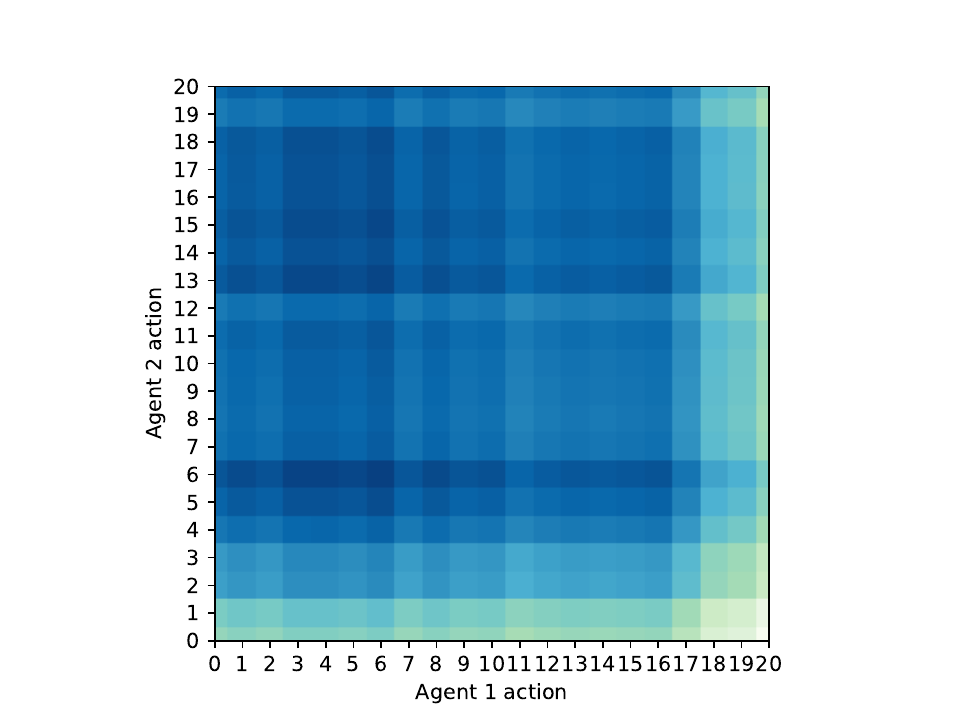}	
		\label{a500}}
	\subfloat[step 1000.]{
		\includegraphics[width=0.25\linewidth]{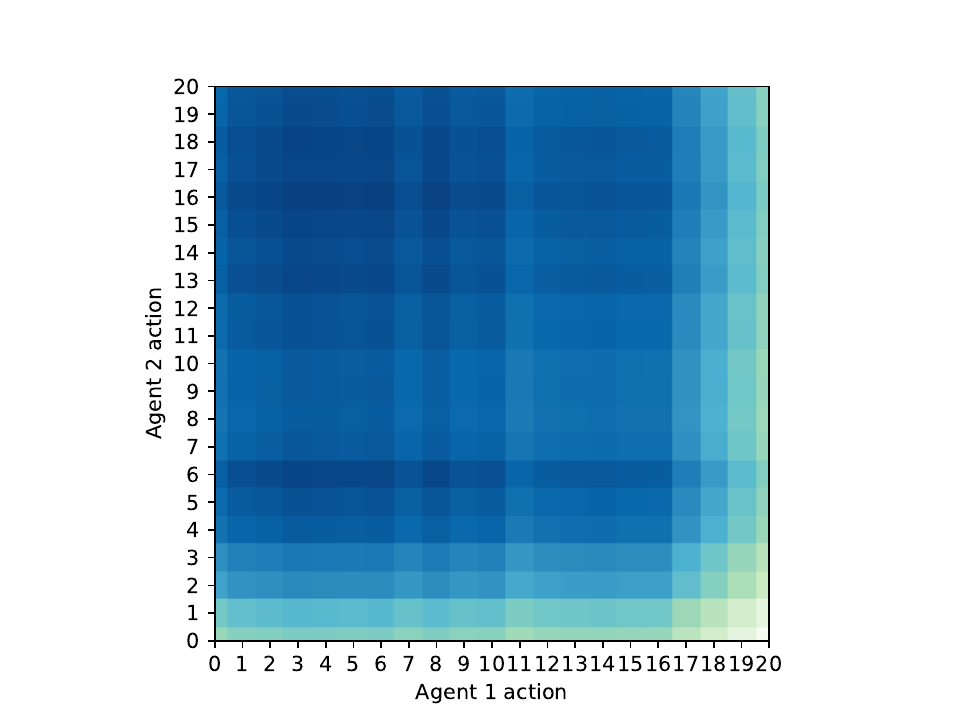}	
		\label{a1000}}
	\subfloat[step 2000.]{
		\includegraphics[width=0.25\linewidth]{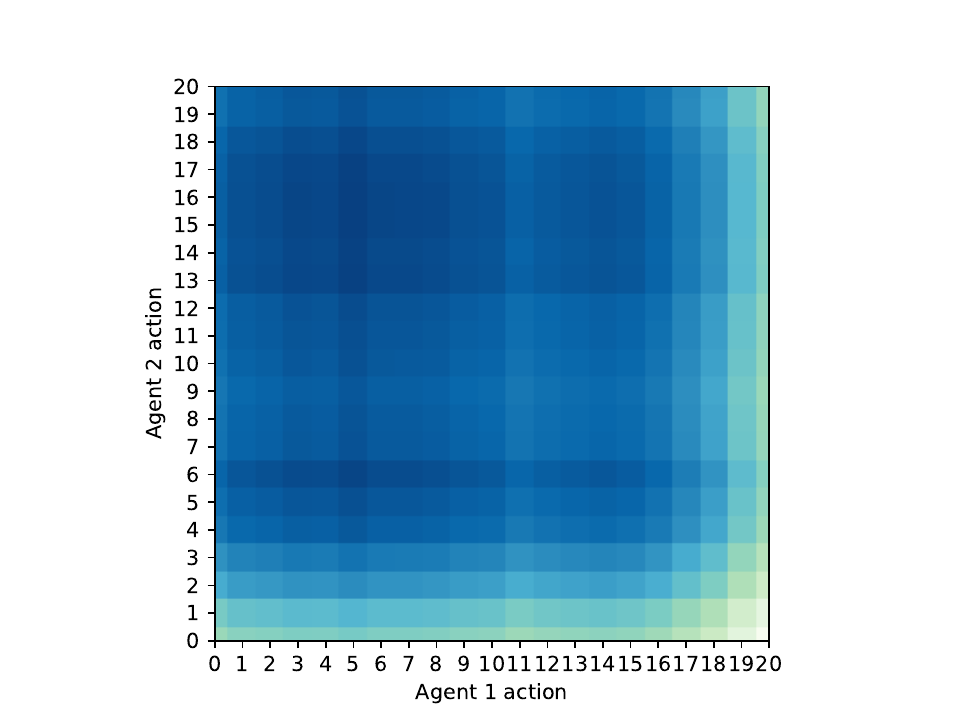}	
		\label{a2000}}
	\caption{Matrix game training process of QFree-Sum.}
	\label{Matrix game2a}
\end{figure*}

\begin{figure*}[!t]
	\centering
	\subfloat[2c\_vs\_64zg.]{
		\includegraphics[width=0.31\linewidth]{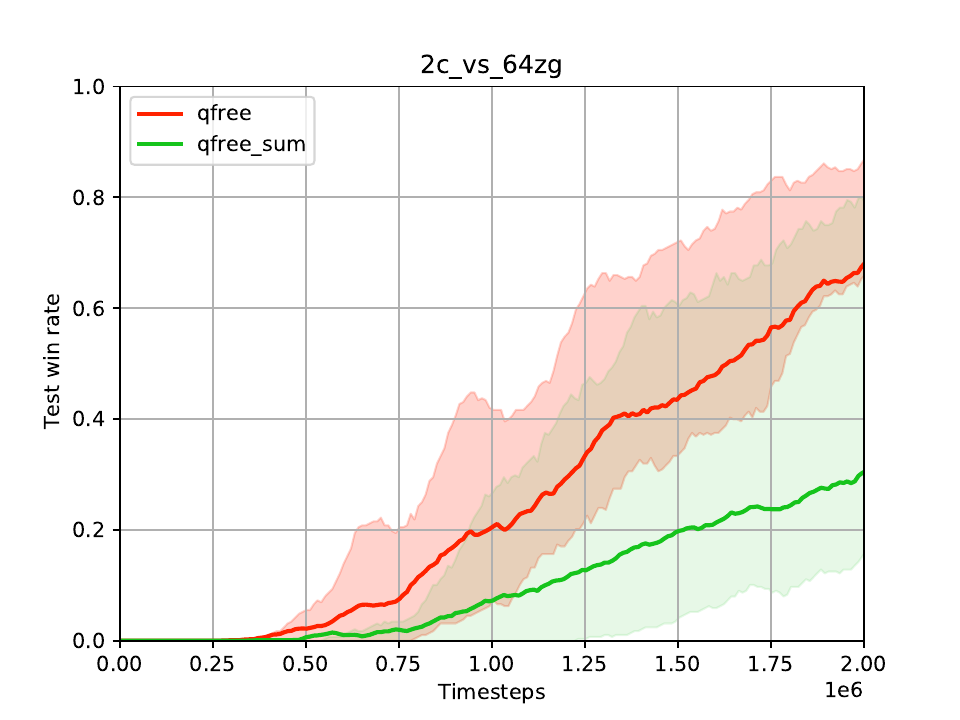}	
		\label{2cvs64zg_s}}
	\hfill  
	\subfloat[5m\_vs\_6m.]{
		\includegraphics[width=0.31\linewidth]{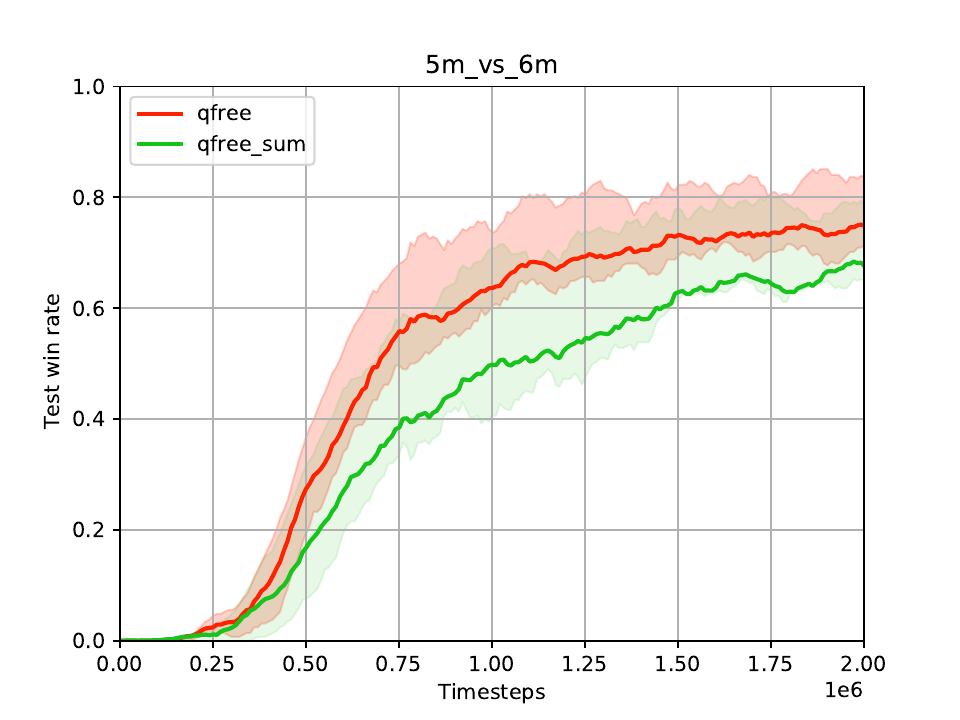}	
		\label{5mvs6m_s}}
	\hfill 
	\subfloat[3s5z\_vs\_3s6z.]{
		\includegraphics[width=0.31\linewidth]{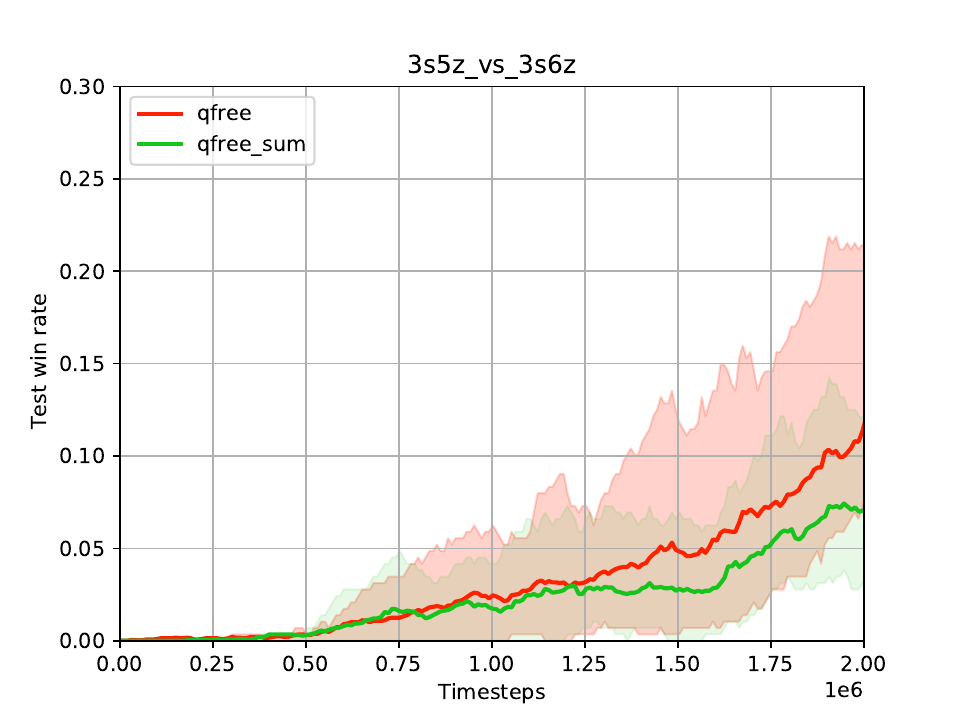}	
		\label{3s5z_vs_3s6z_s}}
	\caption{The test winning rate of QFree and QFree-Sum.}
	\label{fig_SMAC_A}
\end{figure*}	

\begin{figure*}[!t]
	\centering
	\subfloat[2c\_vs\_64zg.]{
		\includegraphics[width=0.31\linewidth]{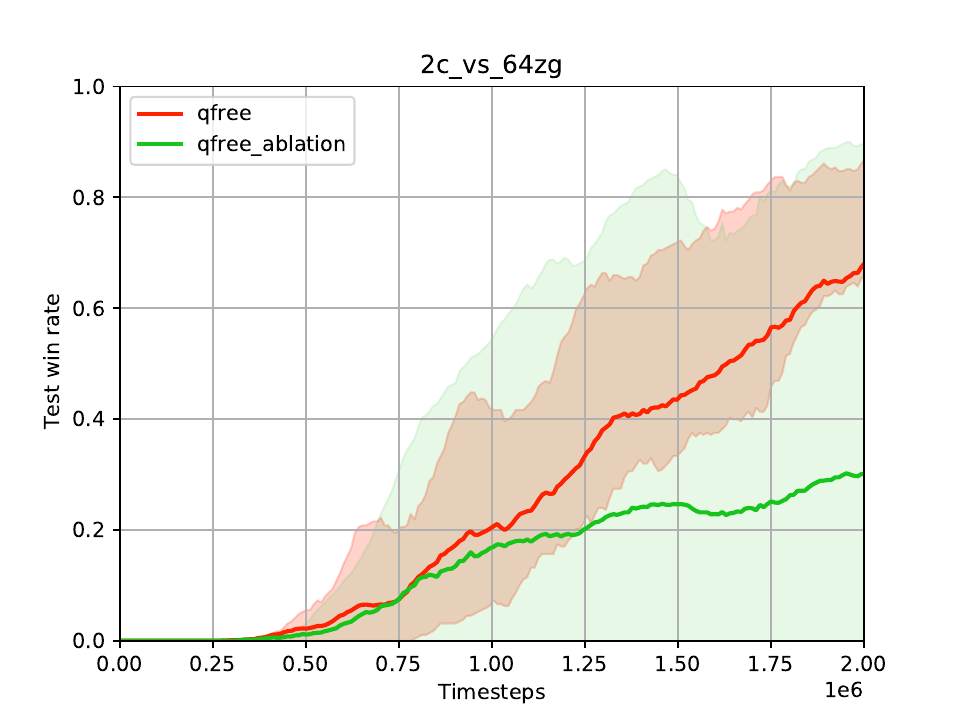}	
		\label{2cvs64zg_ab}}
	\hfill  
	\subfloat[5m\_vs\_6m.]{
		\includegraphics[width=0.31\linewidth]{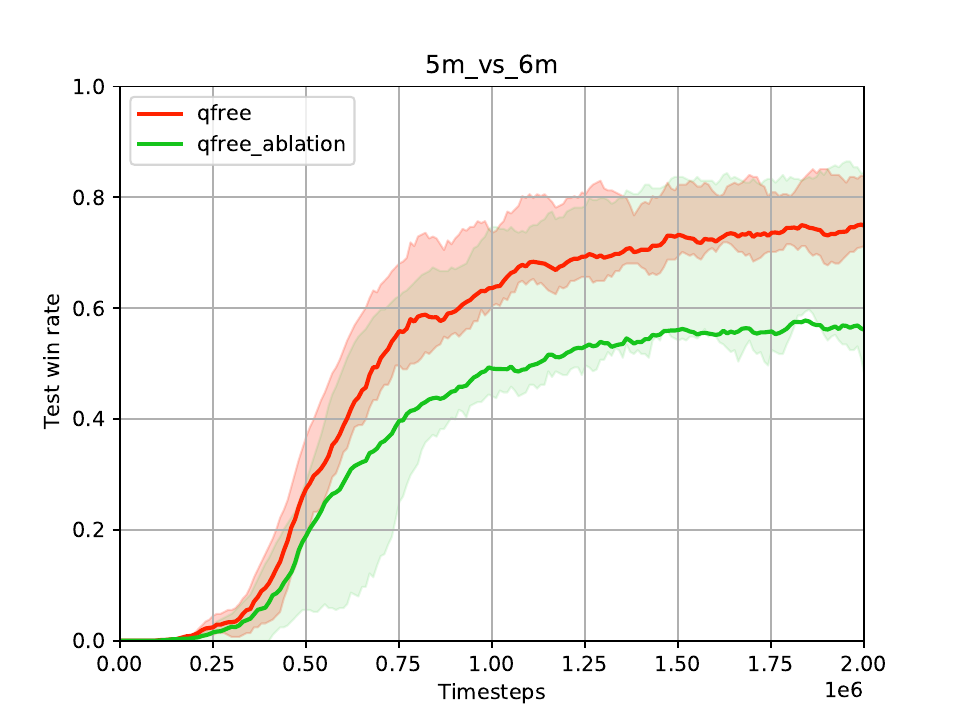}	
		\label{5mvs6m_ab}}
	\hfill 
	\subfloat[3s5z\_vs\_3s6z.]{
		\includegraphics[width=0.31\linewidth]{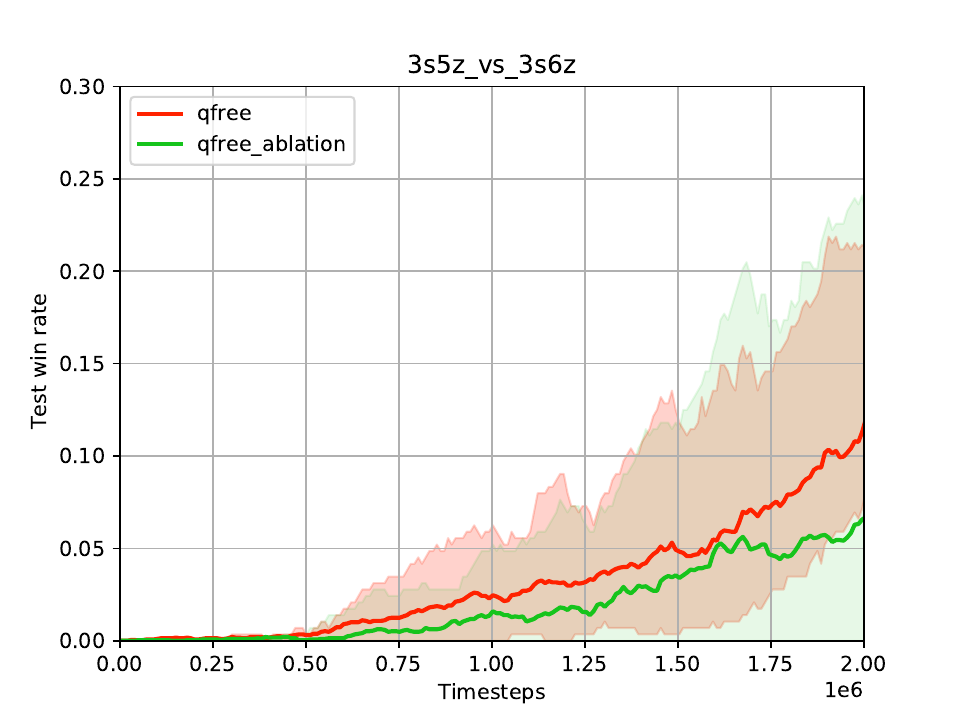}	
		\label{3s5z_vs_3s6z_ab}}
	\caption{The test winning rate of QFree and QFree-Ablation.}
	\label{fig_SMAC_Ab}
\end{figure*}	

In this subsection, we investigate the impact of two feedforward mixing networks (\ref{trans_2}) on the expressive range of the joint action value function. We first transform the mixing network (\ref{trans_2}) into a summation form
\begin{equation}\label{trans_3}
	\begin{aligned}
		V_{tot}(\bm z)=\sum_{i=1}^{n} V_{i}(\bm z),\\
		A_{tot}(\bm z, \bm a)=\sum_{i=1}^{n} A_{i}(\bm z, a_i).
	\end{aligned}
\end{equation}

For clarity, we refer to the algorithm that computes $V_{tot}(\bm z)$ and $A_{tot}$ using (\ref{trans_3}) as QFree-Sum. We conducted experiments using a 21-dimensional matrix game (\ref{matrix}) to evaluate the performance of QFree-Sum. Fig. \ref{Matrix game2a} illustrates the changes in the reward function matrix learned throughout the entire training process of QFree-Sum. Compared to the real reward function depicted in Fig. \ref{Matrix game1}, it appears that QFree is unable to learn the optimal policy since it does not learn the actual reward function matrix. For the experimental results of QFree in Fig. \ref{Matrix game2}, it can be observed that QFree can express a larger range of action joint value function class compared to QFree-Sum, particularly in the non-monotone case. This provides evidence for the efficacy of designed mixing network in expanding the expressive range of the joint action value function class.

In addition, we evaluate the performance of QFree and QFree-Sum in complex environments using three SMAC maps: 2c\_vs\_64zg, 5m\_vs\_6m, and 3s5z\_vs\_3s6z. Similar to the previous experiment, we compute the average winning percentage over 20 training sessions and present them along with their corresponding $75\%$ confidence intervals in Fig. \ref{fig_SMAC_A}. It can be seen that QFree has achieved a higher winning rate in three maps, indicating that QFree has better performance than QFree-Sum in complex environments. Compared with QFree-Sum, QFree has the ability to approximate the more complex relationship between $Q_i$ and $Q_{tot}$, so it has better performance in complex environments.

\subsubsection{Policy Evaluation Algorithm}
	
In this subsection, we aim to demonstrate the necessity of the regularization constraint terms in Algorithm \ref{alg1} through ablation experiments. Specifically, we compare the performance of two algorithms: QFree and QFree-Ablation. The only difference between these algorithms lies in the use of the regularization term in the loss function. It is important to note that all other structural parameters of these algorithms remain consistent. To evaluate their performance, we conduct experiments on three different maps, namely 2c\_vs\_64zg, 5m\_vs\_6m, and 3s5z\_vs\_3s6z. Similar to the previous experiment settings, the results are shown in  Fig. \ref{fig_SMAC_Ab}.
	
From the results presented in Fig. \ref{fig_SMAC_Ab}, it is evident that QFree exhibits a higher average winning rate compared to QFree-Ablation across various map environments. Additionally, the variance of the winning rate for QFree is lower than that of QFree-Ablation, indicating that it is superior to QFree-Ablation in terms of stability. Furthermore, it should be emphasized that due to QFree-Ablation's failure to satisfy the IGM principle, there are instances where it is unable to learn a cooperative policy among the multi-agents, resulting in the $0$ winning rate.
	
In conclusion, the ablation experiment shows that QFree outperforms QFree-Ablation in terms of overall performance. The inclusion of regularization constraint terms in QFree is proved to be effective and the developed advantage function based IGM principle holds.

\section{Conclusion}
\label{Conclusion} 

In this paper, we have proposed a universal value function factorization method for MARL. A novel advantage function based IGM principle has been designed to relax the additional constraints on the value function. On this basis, a new mixing network architecture and an MARL algorithm has been designed to facilitate implementation, where the principle was satisfied by adding a regulation term. Moreover, the effectiveness and advantage of the proposed algorithm have been verified through two experimental tests, demonstrating excellent performance even in complex environments with a vast action state space.

\appendix
\setcounter{equation}{0} 
\setcounter{theorem}{0}
\renewcommand{\theequation}{A\arabic{equation}}
\section{Appendix}
\label{Appendix}

\subsection{complete proofs of Theorem 1}
\label{Appendix A1}
\begin{theorem}
	For the advantage function based IGM principle in (\ref{advantage-based IGM}), let $a^*_i=\mathop{\arg\max}\limits_{a_i \in A}A_{i}(z_i, a_i)$ and $\bm a^*=[a^*_1,a^*_2,...,a^*_n]$. Then the principle (\ref{advantage-based IGM}) holds if and only if the following conditions are satisfied:
	\begin{equation}\label{constraint}
		\begin{aligned}
			\begin{cases}
				A_{tot}(\bm z, \bm a)\leq0 & \bm a\neq\bm a^*,\\
				A_{tot}(\bm z, \bm a)=0 & \bm a=\bm a^*.
			\end{cases}
		\end{aligned}
	\end{equation}
\end{theorem}

\begin{proof}
	Suppose the advantage function class satisfying the advantage function based IGM Principles is
	\begin{equation}\label{a2}
		\overline{\bm A}=\{\overline{A}_{tot} \in \mathbb{R}^{|Z||A^N|}, [\overline{A}_{i} \in \mathbb{R}^{|Z||A|}]_{i=1}^n|\text{satisfy}\: \:(\ref{advantage-based IGM})\},
	\end{equation}
	where $\overline{A}_{tot}$ and $\overline{A}_{i}$ denote the joint and individual advantage function, respectively. Similarly, defining
	\begin{equation}\label{a3}
		\hat{\bm A}=\{\hat{A}_{tot} \in \mathbb{R}^{|Z||A^N|},[\hat{A}_{i} \in \mathbb{R}^{|Z||A|}]_{i=1}^n|\text{satisfy}\: \:(\ref{constraint})\}
	\end{equation}
	as the advantage function class that satisfies the conditions of Theorem $1$, where $\hat{A}_{tot}$ and $\hat{A}_{i}$ are the corresponding joint and individual advantage function, respectively. We will prove that (\ref{a2}) and (\ref{a3}) are mutual sufficient and necessary conditions.
	
	From principle (\ref{advantage-based IGM}) we can obtain
	\begin{equation}\label{advantage-based IGM_a1}
		\mathop{\arg\max}\limits_{\bm a \in A^N}\overline{A}_{tot}(\bm z, \bm a)=
		\begin{pmatrix}
			{\arg\max}_{a_1 \in A}\overline{A}_{1}(z_1, a_1)\\
			\vdots\\
			{\arg\max}_{a_n \in A}\overline{A}_{n}(z_n, a_n)\\
		\end{pmatrix}.
	\end{equation}
	We define $\overline{\bm a}^*=\mathop{\arg\max}\limits_{\bm a \in A^N}\overline{A}_{tot}(\bm z, \bm a)$ and $\overline{a}_i^*=\mathop{\arg\max}\limits_{a_i \in A^N}\overline{A}_{i}(\bm z, a_i)$. From (\ref{advantage-based IGM_a1}) we get $\overline{\bm a}^*=[\overline{a}_1^*,\overline{a}_2^*,...,\overline{a}_n^*]$. According to (\ref{zero}), it can be obtained that
	\begin{equation}\label{constraint_a1}
		\begin{aligned}
			\overline{V}_{tot}(\bm z)=\overline{Q}_{tot}(\bm z, \overline{\bm a}^*).\\
		\end{aligned}
	\end{equation}
	Thus
	\begin{equation}\label{constraint_a2}
		\begin{aligned}
			\overline{A}_{tot}(\bm z, \overline{\bm a}^*)=\overline{V}_{tot}(\bm z)-\overline{Q}_{tot}(\bm z, \overline{\bm a}^*)=0.\\
		\end{aligned}
	\end{equation}
	Due to the definition of the $\arg\max$ operator,  if $\overline {\bm a}\neq \overline{\bm a}^*$ we have
	\begin{equation}\label{constraint_a3}
		\begin{aligned}
			\overline{A}_{tot}(\bm z, \overline {\bm a})  < \overline{A}_{tot}(\bm z, \overline{\bm a}^*)=0.\\
		\end{aligned}
	\end{equation}
	In summary, it can be deduced from Definition 2 that Theorem 1 holds, i.e., $\overline{\bm A} \subseteq \hat{\bm A}$.
	
	Next we will show that Definition \ref{definition2} can be derived. According to $\hat{\bm a}^*=[\hat{a}^*_1,\hat{a}^*_2,...,\hat{a}^*_n]$ in Theorem \ref{theorem1}, we can obtain
	\begin{equation}\label{advantage-based IGM_a2}
		\hat{\bm a}^*=
		\begin{pmatrix}
			{\arg\max}_{a_1 \in A}\hat{A}_{1}(z_1, a_1)\\
			\vdots\\
			{\arg\max}_{a_n \in A}\hat{A}_{n}(z_n, a_n)\\
		\end{pmatrix}.
	\end{equation}
	According to Theorem \ref{theorem1}, we can see that $\hat{A}_{tot}(\bm z, \bm a^*)=0$. By combined with (\ref{Advantage}), we have
	\begin{equation}\label{Advantage_a}
		\begin{aligned}
			\hat{Q}_{tot}(\bm z,  \hat{\bm a}^*)=\hat{V}_{tot}(\bm z).
		\end{aligned}
	\end{equation}
	We can deduce from (\ref{zero}) that
	\begin{equation}\label{zero_a1}
		\begin{aligned}
			\hat{Q}_{tot}(\bm z,  \hat{\bm a}^*)=\mathop{\max}\limits_{\bm a \in A^N}\hat{Q}_{tot}(\bm z, \bm a).\\
		\end{aligned}
	\end{equation}
	Therefore, by using (\ref{zero_a1}) we can get
	\begin{equation}\label{zero_a2}
		\begin{aligned}
			\hat{\bm a}^*=\mathop{\arg\max}\limits_{\bm a \in A^N}\hat{Q}_{tot}(\bm z, \bm a)=\mathop{\arg\max}\limits_{\bm a \in A^N}\hat{A}_{tot}(\bm z, \bm a).
		\end{aligned}
	\end{equation}
	Thus, it is natural to get
	\begin{equation}\label{advantage-based IGM_a3}
		\mathop{\arg\max}\limits_{\bm a \in A^N}\hat{A}_{tot}(\bm z, \bm a)=
		\begin{pmatrix}
			{\arg\max}_{a_1 \in A}\hat{A}_{1}(z_1, a_1)\\
			\vdots\\
			{\arg\max}_{a_n \in A}\hat{A}_{n}(z_n, a_n)\\
		\end{pmatrix},
	\end{equation}
	which means that $\hat{\bm A}\subseteq\overline{\bm A}$. According to the above results, $\hat{\bm A}$ and $\overline{\bm A}$ are mutually sufficient and necessary, i.e., $\hat{\bm A}$ and $\overline{\bm A}$ are equivalent ($\hat{\bm A} \Leftrightarrow \overline{\bm A}$). Hence, the advantage function based IGM principle and the proposed one in Theorem \ref{theorem1} are equivalent, which completes the proof. $\hfill\Box$
\end{proof}

\end{document}